\documentclass{article}
\usepackage{ijcai22}

\usepackage{times}

\usepackage{soul}
\usepackage{url}
\usepackage{hyperref}
\usepackage[utf8]{inputenc}
\usepackage[small]{caption}
\usepackage{graphicx}
\usepackage{amsmath}
\usepackage{mathtools}
\usepackage{booktabs}
\usepackage{amsmath, nccmath}
\usepackage{placeins}
\usepackage{afterpage}
\renewcommand{\paragraph}{\paragraphsmall}

\usepackage{amssymb,amsmath,amsthm}
\newtheorem{theorem}{Theorem}

\newcommand{\paragraphsmall}[1]{\textbf{{#1} $\;$}}

\DeclarePairedDelimiterX{\dotp}[2]{\langle}{\rangle}{#1, #2}

\newcommand{\kmax}{C}

\newcommand{\be}{\begin{equation}}
\newcommand{\ee}{\end{equation}}

\renewcommand{\AA}{{\mathcal{A}}}

\newcommand{\norm}[1]{\left\| #1\right\|}                               %

\renewcommand{\epsilon}{\varepsilon}

\newcommand{\RR}{{\mathcal R}}

\DeclareMathOperator*{\E}{\mathbb{E} \,}

\DeclarePairedDelimiter{\ceil}{\lceil}{\rceil}

\newcommand{\neuron}{r}
\newcommand{\edge}{j}

\newcommand{\epsilonLayer}[1][\ell]{\epsilon_{#1}}

\newcommand{\SampleComplexityDeltaLayers}[1][\epsilonLayer]{\ceil*{32  \, L^2 (\Delta_r^{\ell \rightarrow})^2 \, \kmax \log (8 \eta / \delta) \,  \epsilon^{-2} \, \sum_{\edge \in \Wpm} s_\edge }}

\newcommand{\Wpm}{\II}

\newcommand{\II}{\mathcal{I}}

\newcommand{\WWRowCon}[1][\neuron]{w}
\newcommand{\WWHatRowCon}[1][\neuron]{{\hat{w}}}

\newcommand\Reals{\mathbb{R}}
\newcommand\PP{\mathcal{P}}

\newcommand\TT{\mathcal{T}}

\newcommand\VV{\mathcal{V}}

\newcommand{\F}[1][\hat w]{F_\lambda(\PP, #1)}

\renewcommand{\epsilon}{\varepsilon}

\newcommand{\grad}[1][\hat w]{\norm{\nabla \F}_2}

\usepackage[utf8]{inputenc} 
\usepackage[T1]{fontenc}    
\usepackage{url}            
\usepackage{booktabs}       
\usepackage{amsfonts}       
\usepackage{nicefrac}       
\usepackage{microtype}      
\usepackage{nicefrac}
\usepackage{wrapfig}
\usepackage{algorithm}
\usepackage{algorithmic}
\usepackage{graphicx}
\usepackage[export]{adjustbox}
\usepackage{subcaption}

\usepackage{float} 
\usepackage{amsmath,amsfonts,amsthm,amssymb, nccmath}
\allowdisplaybreaks 
\usepackage{bbm}
\usepackage{framed}
\usepackage{enumerate}
\usepackage{mathtools}
\usepackage{scalefnt}
\usepackage{changepage}
\usepackage{xargs}

\usepackage{comment}
\usepackage{setspace}
\usepackage{verbatim}
\usepackage{xcolor}
\usepackage[font=small]{caption}
\usepackage{scalefnt}
\usepackage{xr}
\usepackage[shortlabels]{enumitem}
\usepackage{placeins}
\usepackage{color}
\usepackage{footmisc}
\usepackage{nicefrac}

\usepackage{thmtools}
\usepackage{thm-restate}
\usepackage{multirow}
\usepackage{amsfonts}
\usepackage{bbm}

\title{Bandit Sampling for Multiplex Networks}

\author{
Cenk Baykal$^1$\footnote{This work was done when C. Baykal was at J.P. Morgan AI Research.}\and
Vamsi K. Potluru$^{2}$\and
Sameena Shah$^2$\and
Manuela M. Veloso$^2$ \\
\affiliations
$^1$Google Research\\
$^2$J.P. Morgan AI Research\\
\emails
baykal@alum.mit.edu,
\{vamsi.k.potluru,sameena.shah,manuela.veloso\}@jpmchase.com
}
\begin{document}

\maketitle

\begin{abstract}
Graph neural networks have gained prominence due to their excellent performance in many classification and prediction tasks. In particular, they are used for node classification and link prediction which have a wide range of applications in social networks, biomedical data sets and financial transaction graphs. Most of the existing work focuses primarily on the \emph{monoplex} setting where we have access to a network with only a single type of connection between entities. However, in the \emph{multiplex} setting, where there are multiple types of connections, or \emph{layers}, between entities, performance on tasks such as link prediction has been shown to be stronger when information from other connection types is taken into account. We propose an algorithm for scalable learning on multiplex networks with a large number of layers. The efficiency of our method is enabled by an online learning algorithm that learns how to sample relevant neighboring layers so that only the layers with relevant information are aggregated during training. This sampling differs from prior work, such as MNE, which aggregates information across \emph{all} layers and consequently leads to computational intractability on large networks. Our approach also improves on the recent layer sampling method of \textsc{DeePlex} in that the unsampled layers do not need to be trained, enabling further increases in efficiency. We present experimental results on both synthetic and real-world scenarios that demonstrate the practical effectiveness of our proposed approach.
\end{abstract}

\section{Introduction}
\label{sec:introduction}

Graph Neural Networks (GNNs)~\cite{bruna2013spectral,duvenaud2015convolutional,gilmer2017neural} have been unprecedentedly successful in many high-impact applications, ranging from drug discovery to financial and social network analyses. Most of the prior work on GNNs has focused on the monoplex setting where we have access to a network with only a single type of connection between entities. However, in many real-world settings two nodes may be connected in more than one way. For example, a person may be a part of multiple social networks such as Facebook, Twitter, and/or Instagram, among others. In this case, the person's connections on, e.g., Facebook, may reveal information about their connections on other platforms, and more generally, be intricately linked with their graphical connectivity on other platforms. A \emph{multiplex network} is a representation of such connectivity.

A multiplex network is composed of multiple \emph{layers}, i.e., sub-networks where only one type of connection is present. In the context of the previous example, the Facebook network is a single layer of the multiplex network. The multiple layers (Facebook, Twitter, Instagram) then constitute the multiplex as a whole. Given the increasing need to model and learn from these intricate connections, the question of how to best model and train multiplexes has recently gained significant interest. This question is further motivated by the immense computational complexity of learning on graphs comprised of multiple sub-networks, each of which is computationally intensive to train in itself.

\subsection{Background}

In our work, we focus on the problem of \emph{computationally-efficient} link prediction in the multiplex setting. Prior work in scalable training multiplex networks includes Multiplex Network Embedding (\textsc{MNE})~\cite{zhang2018scalable} and \textsc{DeePlex}~\cite{potlurudeeplex} which extends the monoplex setting of SEAL~\cite{zhang2018scalable}. Note that there is a rich history of literature for link prediction problems in the monoplex settings encompassing both unsupervised and supervised approaches~\cite{libennowell-04,adamic:01,katz:53,grover2016node2vec,kipf2016variational,perozzi2014deepwalk,tang2015line,zhang2018link,yun2021neo}.

\paragraph{\textsc{MNE}} In \textsc{MNE}~\cite{zhang2018scalable} the idea is to learn a base embedding that utilizes information regarding \emph{all} the links in \emph{all} the sub-networks and individual node embeddings for each of the layers. Concretely, 
given a network of layers $G_1, \ldots, G_L$ where we have $L$ layers and $G_i = (N_i, E_i)$ corresponding to sets of $N_i$ nodes and $E_i$ edges, MNE learns a node embedding 
    $v^{i}_n = b_n + w^i \dot X^{i^T} u_n^i$
where $X^i \in R^{s \times d}$, $b_n$ correspond to the base node embedding and $u_n^i$ to the individual node embedding for the layer, respectively, and $w^i$ are the learned weights.
The matrices $X^{i}$ account for the high-dimensional global embeddings with the lower dimensional individual embeddings. The model is learned by utilizing random walks on each layer type to generate a sequence of nodes and a skip-gram algorithm is used to learn the embeddings. Although \textsc{MNE} has been successful on real-world data sets, it is not able to handle multiplexes with many layers in a computationally efficient way because it aggregates information over all layers -- including those that may not be relevant for the layer under consideration.

\paragraph{\textsc{DeePlex}} 
Aggregating information across all the layers as in MNE not only necessitates training all of the layers at each time step, but also incorporating the dense embeddings of all other layers in training a model such as a feed-forward neural network; this means that a larger number of model parameters have to be learned due to the increase in the input size to the networks. To overcome this shortcoming, \textsc{DeePlex} considers sampling $k$-nearest layers for a suitably chosen $k$ while training the network. The premise of this approach is that not all layers will be relevant for the current layer's embedding, and those layers whose embeddings are most similar to the current one should be sampled. The authors hence use a NeuralSort operator to obtain a continuous relaxation of the output of a sorting operator, and subsequently train the network (and the sort operator) in the standard way to learn both the architecture and the sorted order of neighboring layers. 

\textsc{DeePlex} has shown promise as a scalable approach for multiplexes. However, its main focus was on the inference-time complexity where we are interested in predicting at a layer based without utilizing all of the (embedding) information from the entire multiplex network. Additionally, \textsc{DeePlex} suffers from the computational complexity of learning a k-nearest layers model for each layer, which can be prohibitively expensive when the number of layers is large.

\paragraph{Scalability via Regret Minimization}
A related line of work has investigated the use of bandit algorithms to tame the complexity of Graph Neural Networks by sampling neighboring \emph{nodes} during the aggregation step which entails a sum over all neighbors~\cite{liu2020bandit,zhang2021biased}. In~\cite{liu2020bandit}, the authors approximate the aggregation across all neighbors via a small subset of neighbors that is sampled according to a variant of the Exponential-weight algorithm for Exploration and Exploitation (Exp3) algorithm~\cite{auer2002finite}. However, these works impose restricting assumptions on the norm of the embeddings and require a priori knowledge of instance-specific parameters. They also employ fixed learning rates that are based on a worst-case analysis. Our approach, on the other hand, is adaptive and parameter-free and does not impose any assumption on the embeddings or the losses. This enables the use of any user-specified loss metric with our method. This also leads to instance-specific bounds that are strictly better.

\subsection{Our Contributions}

In this work, we focus on increasing the efficiency of training and predicting with multiplex networks via adaptive sampling of the relevant layers. Namely, we use an online learning algorithm for each layer that \emph{adaptively} learns a sampling distribution over the neighboring layers. Unlike \textsc{MNE}, our approach selectively hones in on the information from relevant layers, rather than the full set of information available, in order to obtain speedups in both training and inference. Moreover, unlike \textsc{DeePlex}, our approach does not require learning a k-nearest layers model for each layer, hence, it is applicable even during training.

In particular, this paper contributes the following:

\begin{itemize}
    \item A formulation of the multiplex layer sampling problem as an online learning problem with partial information
    \item An algorithm for adaptive identification and sampling of important layers of a multiplex network for training 
    \item Evaluations and comparisons on real-world data sets that demonstrate the practicality of our approach
\end{itemize}

\section{Problem Definition}
\label{sec:problem-definition}

Consider a multiplex network $\mathcal G$ with $L$ layers (subnetworks). This means that for any layer $i \in [L]$, we have that layer $i$ has $n = L - 1$ neighboring nodes whose embedding information can be used towards learning the embedding of layer $\ell$ at each time step $t$. Suppose we have $T$ training time steps. The embedding of each layer $j \in [L]$ at each time step $t \in [T]$ is denoted as $v_j^{(t)}$. 

For a fixed layer $i \in [L]$, we consider the problem of sampling $k$ relevant layers at each time step of the training. Subsequently, this procedure will be extended for sampling neighbors of all $L$ layers. Here, we assume that at time step $t \in [T]$ the \emph{relevancy} of a neighboring layer $j \in [n]$ can be quantified by a loss $\ell_{t, j} \ge 0 $. The smaller the loss, the more relevant the layer $j$ towards the embedding of layer $i$ under consideration. Under this general setting, the question is what should the sampling distribution of the neighboring layers be so that we only consider the relevant layers?

This dilemma is known as the exploration-exploitation trade-off and has been well-studied in the multi-armed bandits and online learning with partial information settings. Here, one way to resolve this interdependence is to formulate the problem as one of minimizing \emph{regret} relative to the best action that we could have taken in hindsight. Hence, in the context of the setting above, we would like to generate a sequence of probability distributions $(p_1, \ldots, p_T)$ such that our regret relative to the best distribution in hindsight is minimized:
\begin{align}
\label{eqn:regret}
\mathrm{Regret}\left((\ell_t, p_t)_{t \in [T]} \right) = \sum_{\tau} \dotp{p_\tau}{\ell_\tau} - \min_{p \in \Delta} \sum_{\tau} \dotp{p}{\ell_\tau},
\end{align}
where $\Delta = \{p \in [0,1]^n: \sum_{k \in [n]} p_k = k\}$. In this paper, we focus on  $k = 1$ and discuss extensions to $k > 1$ in Sec.~\ref{sec:discussion}.

\section{Method}
\label{sec:method}

Multi-armed Bandits (MAB) and online learning with partial information literature is rich with algorithms that attempt to minimize the regret expressed in Eq.~\ref{eqn:regret}~\cite{orabona2019modern,luo2014drifting,luo2015achieving,gaillard2014second,bubeck2012regret,koolen2015second}. In this paper, we focus on a subclass of online learning algorithms that can handle the \emph{partial information setting}, where the entire loss vector $\ell_t \ge 0$ is not visible, but rather, only the losses of the $k$ elements we sample according to $p_t$ are visible.  Here, for a given layer $i$ the \emph{loss} of a neighboring layer $j$ at time step $t$ of the training process, $\ell_{t,j}$, is a user-specified measure of whether or not layer $j$'s embedding would be helpful for the training of layer $i$.

In Sec.~\ref{sec:results} we outline and experiment with various definitions of loss. In this context, it appears that we can apply a standard bandit algorithm or a variant of previous work that was focused on bandit sampling of \emph{nodes}~\cite{liu2020bandit,zhang2021biased}. However, one issue is that virtually all of the prior work in taming the exploration-and-exploitation trade-off assumes that the losses are bounded in the interval $[0,1]$, which is not guaranteed to be the case for all user-specified metrics. Although we could use heuristics to get around this limitation by, e.g., capping the losses to be at most 1 and applying standard algorithms, this may yield bad results in practice -- for instance when all losses are greater than 1, they become indistinguishable after the capping. This assumption also imposes restrictions on the applicability of our algorithm with any user-specified loss metric.


\begin{algorithm}
\caption{\textsc{Exp3$^+$}}
\label{alg:adaprod}
\begin{spacing}{1}
\begin{algorithmic}[1]
\STATE $L \gets \vec{0} \in \Reals^{n}$ \COMMENT{Cumulative Loss}; \, $E_0 \gets 1$ \COMMENT{Loss range}; \, $V_0 \gets 0$ \COMMENT{Cumulative variance}
\FOR{each round $t \in [T]$} 
    \STATE $\eta_t \gets \min \{1 / E_{t-1}, \sqrt{\nicefrac{\log n}{ V_{t-1}}} \}$ \text{if $t \ge 2$, $0$ o.w.}
    \STATE $p_{t,i} \gets \exp(\eta_{t} \, L_{i} )$ \, \, for each $i \in [n]$
    \STATE $p_{t,i} \gets \nicefrac{p_{t,i}}{ \sum_{j \in \AA_t} p_{t,j}}$ for all $i$ \COMMENT{Normalize} 
    \STATE Random draw $i_t \sim p_t$
    \STATE \text{Adversary reveals $\ell_{t,i_t}$ and we suffer loss $\ell_{t,i_t}$} \\
    \STATE Construct unbiased estimate $\hat \ell_{t}$ $$\forall{i \in [n]} \qquad \hat \ell_{t,i} \gets \begin{cases} \nicefrac{\ell_{t,i_t}}{p_{t,i_t}} & \text{if $i = i_t$}, \\ 
    0 & \text{otherwise}
    \end{cases}$$
    \STATE $L \gets L + \hat \ell$ ; $\Delta_t \gets \max_{i, j \in [n]} | \hat \ell_{t,i} - \hat \ell_{t,j}|$
    \STATE $E_t \gets \max\{ E_{t-1}, 2^{k}\}$ where $k = \ceil{\log_2 \Delta_t }$
    \STATE $V_t \gets V_{t-1} + \left(\dotp{\hat \ell_t^2 }{p_t} - \dotp{\hat \ell_t}{p_t}^2\right)$

\ENDFOR 
\end{algorithmic}
\end{spacing}
\end{algorithm}

 To overcome these challenges, we adapt the approach of~\cite{cesa2007improved} (see Alg.~\ref{alg:adaprod}) which relies on a time-varying learning rate that can adapt to any range of losses $E$~\cite{sani2014exploiting,auer2002finite} and leads to an instance-specific variance bound. Here, the main idea is to adaptively update the value of a scaling factor $E_t$ which essentially denotes the range of the losses we have seen so far as a power of 2 (Line 10). The learning rate $\eta_t$ is then scaled automatically as a function of $E_{t-1}$ and the variance of the losses we have seen so far to account for this update (Line 3). The result is an adaptive algorithm that scales to the range of losses seen in practice.

To put it all together in the case of multiplex layer sampling, the idea is to have a separate instance of Alg.~\ref{alg:adaprod} for each layer of the network and update the sampling distribution accordingly as shown in Alg.~\ref{alg:adaprod}. This means that we will have $L$ separate algorithm instances, one for each layer, that will learn the sampling distribution over the $[n] = [L - 1]$ neighboring layers for that specific layer.

It can be shown by adapting the analysis in~\cite{cesa2007improved} that we have the following bound on the regret defined by Eqn.~\ref{eqn:regret}. We emphasize here that unlike the work of~\cite{liu2020bandit,zhang2021biased}, our approach is fully parameter-free and does not impose any assumptions on the magnitude or range of the losses. Additionally, in the context of $\max_{t,i} \ell_{t,i} \leq 1$ as in prior work, the bound below is strictly better since $\VV_T \leq T/4$ in this case.

\begin{theorem}
    The regret $\mathrm{Regret}((\ell_t, p_t)_{t \in [T]})$ of Alg.~\ref{alg:adaprod} over $T$ time steps is bounded above by
    \begin{align*}
       6 \sqrt{\VV_T \log n} +  10M \log n
    \end{align*}
    where $\VV_T$ is the cumulative variance, i.e.,
    $
    \VV_T = \sum_{t \in [T]} \sum_{i \in [n]} p_{t,i}  (\ell_{t,i} - \dotp{\ell_{t}}{p_t})^2,
    $
    $M = \max_{t \in [T]}  \max_{i, j \in [n]} | \ell_{t,i} - \ell_{t,j}|$, and $n$ is the number of neighboring layers to sample from.
\end{theorem}
\begin{proof}
The theorem follows from the analysis of~\cite[Theorem 3]{cesa2007improved}, the only exception here is that we deal with \emph{losses} rather than payoffs in the \emph{partial information} setting where we deal wiyh $\hat \ell$. Here, we derive the result for our application. The main idea is to bound the regret by an appropriately defined potential function whose aggregation over the $T$ iterations can be adequately controlled.

Let $\bar \ell_t = \dotp{\hat \ell_t}{p_t}$ denote the expected loss at round $t$ and define the potential function
$
\Phi_t = \frac{1}{\eta_t} \log \left (\sum_{i \in [n]} p_{t,i} \exp \left(-\eta_{t} (\hat \ell_{t,i} - \bar \ell_t) \right) \right).
$
Invoking a technical helper lemma~\cite[Lemma 2.5]{cesa2006prediction} (or~\cite[Lemma 3]{cesa2007improved})\footnote{Using the translation $r_{t,i} = R - \hat \ell_{t,i}$ to go from losses to rewards to apply the result, and back, where $R = \max_{t} \max_{i} \hat \ell_{t,i}$}, we obtain for the regret after $T$ iterations
\begin{align*}
\mathrm{Regret}((\hat \ell_t, p_t)_{t \in [T]}) 
&\leq \frac{2 \log n}{\eta_{T+1}} + \sum_{t \in [T]} \Phi_t = (A) + (B)
\end{align*}

Hence, all that remains is to bound the terms (A) and (B). The bound on (A) follows by definition of $\eta_{T}$, hence
$
(A) \leq 2 \max\{ E \log n , \sqrt{V_T \log n } \},
$
where $E = \max_{t \in [T]}  \max_{i, j \in [n]} | \hat \ell_{t,i} - \hat \ell_{t,j}| \ge E_{T}$.

Bounding (B) requires bounding the individual potentials $\Phi_t$ and this turns out to be much more involved. To that end, we will use the following observation (obtained by inspection of the derivative). The ratio $(e^x - (1 + x))/x^2$ is increasing for all $x \in [-1, 1]$ and is equal to $(e-2)$ at $x =1$. This implies that for all $x \in [-1,1]$, 
$
e^x \leq 1 + x + (e-2)x^2.
$
Let $x_{t,i} = -\eta_{t} (\hat \ell_{t,i} - \bar \ell_t)$. There are two cases to consider.

If $E_{t} = E_{t-1}$, then we have $x_{t,i} \in [-1,1]$ by definition of the learning rate and $E_t$. Hence, we can apply the inequality above on the exponential function to obtain
\begin{align*}
\Phi_t 
&\leq \frac{1}{\eta_t} \log \left (\sum_{i \in [n]} p_{t,i} \left( 1 + x_{t,i} + (e-2) x_{t,i}^2\right)  \right) \\
&\leq (e-2) \eta_t \sum_{i \in [n]} p_{t,i} (\hat \ell_{t,i} - \bar \ell_t)^2,
\end{align*}
which follows by $ \sum_{i \in [n]} p_{t,i} x_{t,i} = \eta_t  \sum_{i \in [n]} p_{t,i} (\hat \ell_{t,i} - \bar \ell_t) = 0$ and the last inequality by $1 + x \leq e^x$.

If $E_t \neq E_{t-1}$, we can no longer use the upper bound on the exponential, however, we trivially have
$
\Phi_t \leq \frac{1}{\eta_t} \log \left (\max_{i \in [n]} \exp(x_{t,i}) \sum_{i \in [n]} p_{t,i}  \right) \leq E_t.
$

To deal with the two differing cases, let $\TT \subset [T]$ denote the set of time points such that $E_t \neq E_{t-1}$. Then,
\begin{align*}
    \sum_{t \in [T]} \Phi_t 
    &\leq  \sum_{t \in \TT} E_t + (e-2) \sum_{t \notin \TT} \eta_t \underbrace{ \sum_{i \in [n]} p_{t,i}  (\hat \ell_{t,i} - \bar \ell_{t})^2}_{= \hat \VV_t - \hat \VV_{t-1}},
\end{align*}
where $\hat \VV_t = \sum_{t' \in [t]} \sum_{i \in [n]} p_{t',i}  (\hat \ell_{t',i} - \bar \ell_{t'})^2$ is the cumulative variance up to time point $t$.

Now consider a collection of $r$ regimes $\RR_1, \ldots, \RR_r$ where each regime $\RR_i = [a_i, b_i]$ denotes the set of discrete time steps $a_i, a_i + 1, \ldots, b_i$ such that $E_{a_i} = E_{a_i+1} \cdots = E_{b_i}$. Consider an arbitrary regime $\RR$ with start and end points $a$ and $b$ and let $c \in [a,b]$ denote the first integer time step when $\hat \VV_c > E_a^2/4$. Then, using $\hat \VV_{c} \leq \hat \VV_{c-1} + E_a^2/4$ by Popoviciu's inequality on variances and $\hat \VV_{c-1} \leq E_a^2/4$, $V_c \leq E_a^2/2$. This yields
$
    \sum_{t = a}^b \eta_t ({\hat \VV_t - \hat \VV_{t-1}}) \leq E_a/2 + \sqrt{\log n} \sum_{t = c+1}^b \frac{{\hat \VV_t - \hat \VV_{t-1}}}{\sqrt{V_{t-1}}}.
$
We then observe that since $\hat \VV_{t} \leq \hat \VV_{t-1} + E_a^2/4$ and $\hat \VV_{t-1} \ge E_a^2/4$ for $t \ge c + 1$, we have $\hat \VV_t \leq 2 \hat \VV_{t-1}$, and thus $\nicefrac{{\hat \VV_t - \hat \VV_{t-1}}}{\sqrt{V_{t-1}}} \leq (\sqrt{2} + 1) (\sqrt{\hat \VV_t} - \sqrt{\hat \VV_{t-1}})$.
Continuing from above and using a telescoping argument and $\hat \VV_c \ge \hat \VV_a$, we obtain
$$
\sum_{t = a}^b \eta_t ({\hat \VV_t - \hat \VV_{t-1}}) \leq E_a/2 + 4 \left(\sqrt{\hat \VV_b} - \sqrt{\hat \VV_a}\right) \sqrt{\log n}.
$$
Taking the sum over all $r$ regimes and another telescoping arguments yields $\sum_{i \in [r]} E_{a_i}/2 + 4 \sqrt{\hat \VV_T \log n}$.

All that remains to bound the sum of potentials is to observe that $\sum_{i \in [r]} E_{a_i} = \sum_{t \in \TT}$ and that $\sum_{i \in [r]} E_{a_i} \leq \sum_{i = -\infty}^{\ceil{\log_2 E}} 2^i \leq 4 E$.
Putting it all together, we obtain
$$
\mathrm{Regret}((\hat \ell_t, p_t)_{t \in [T]}) \leq  6E + 4 E \log n + 6 \sqrt{ \hat \VV_T \log n}.
$$
The final step is to bound $\mathrm{Regret}(( \ell_t, p_t)_{t \in [T]})$ on the entire loss vector sequence $\ell_t$ rather than the partially observed losses $\hat \ell_t$. To do so, note that conditioned on actions in rounds $1, \ldots, t-1$, $\hat \ell_t$ is an unbiased estimator for $\ell_t$ by construction. The theorem then follows using linearity of expectation to obtain $\mathrm{Regret}((\ell_t, p_t)_{t \in [T]})  = \E[\mathrm{Regret}((\hat \ell_t, p_t)_{t \in [T]})]$, the above bound, and Jensen's inequality.
\end{proof}

\begin{figure*}[htb!]
  \centering
  \begin{minipage}[t]{0.33\textwidth}
  \centering
 \includegraphics[width=1\textwidth]{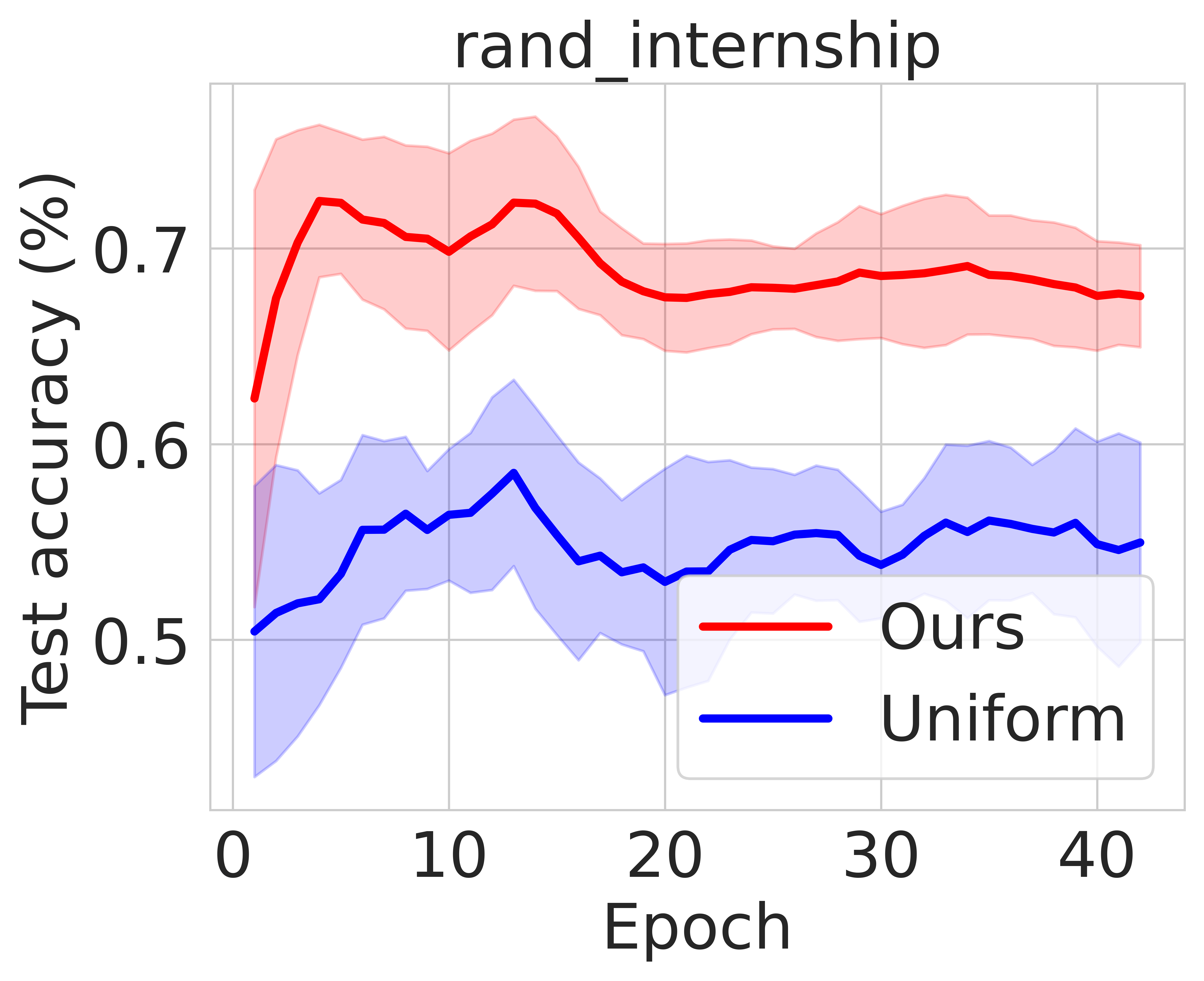}
  \end{minipage}%
  \hfill
  \begin{minipage}[t]{0.33\textwidth}
  \centering
 \includegraphics[width=1\textwidth]{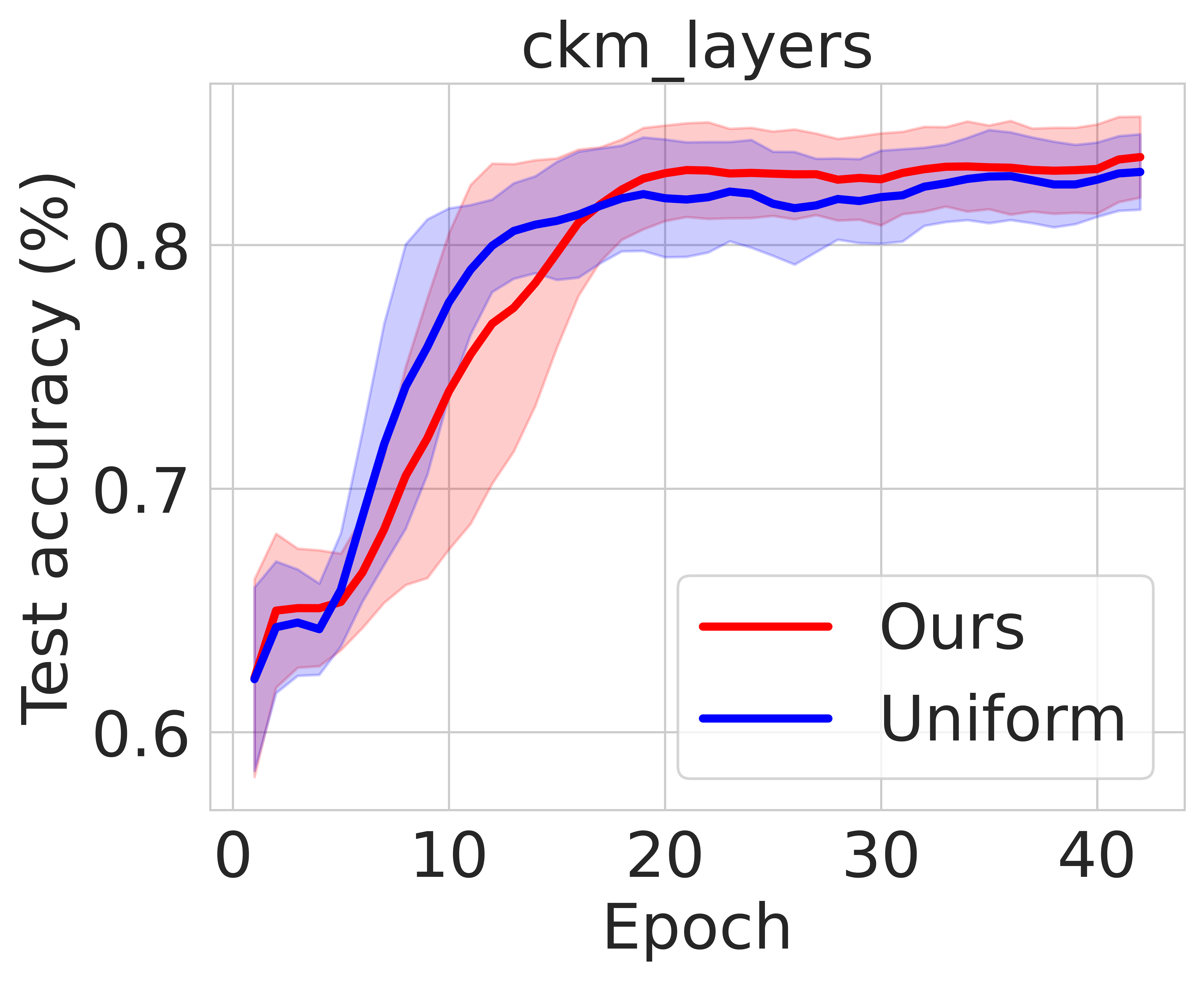}
  \end{minipage}%
  \hfill
    \begin{minipage}[t]{0.33\textwidth}
  \centering
 \includegraphics[width=1\textwidth]{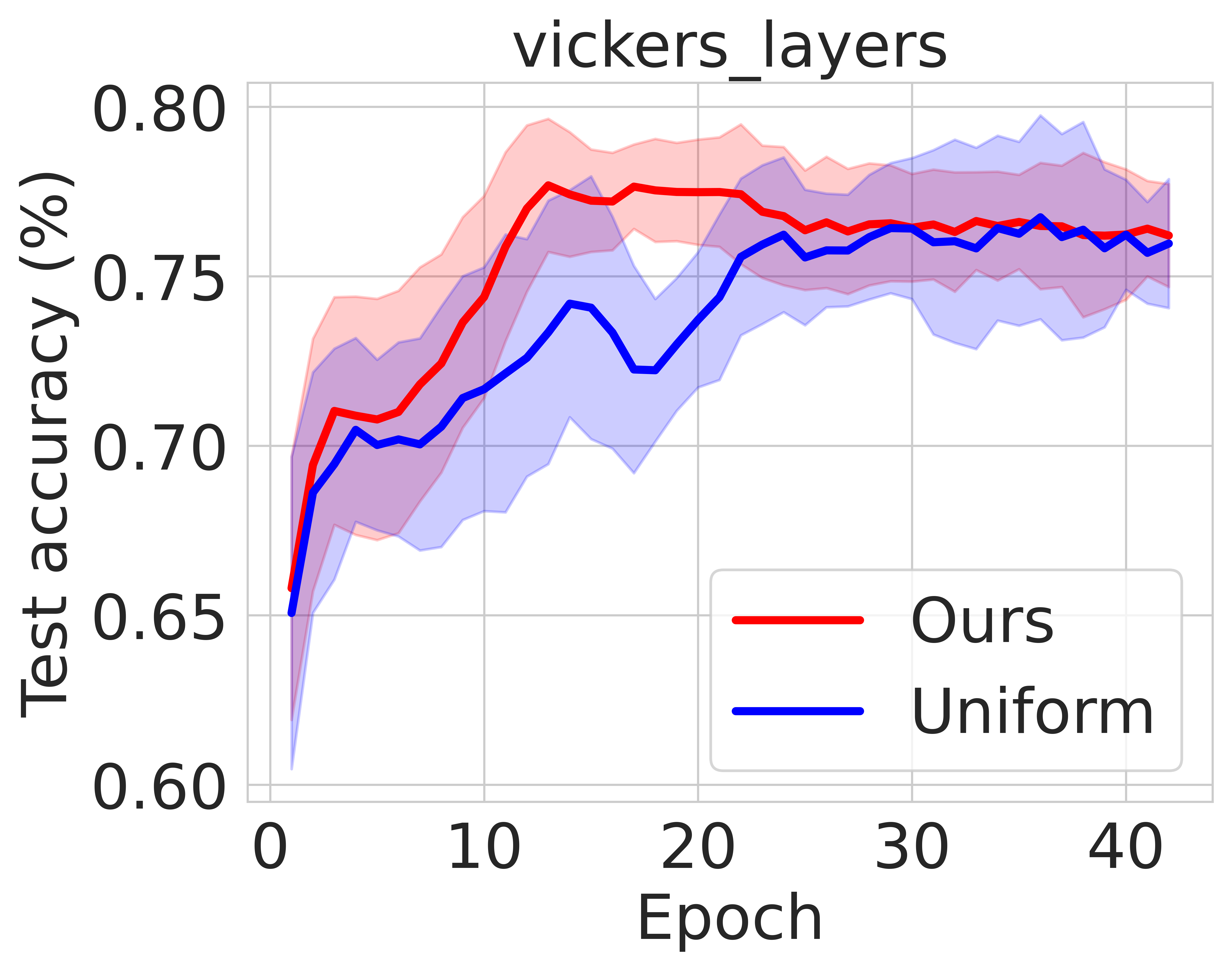}
  \end{minipage}
  
    \begin{minipage}[b]{0.33\textwidth}
  \centering
 \includegraphics[width=1\textwidth]{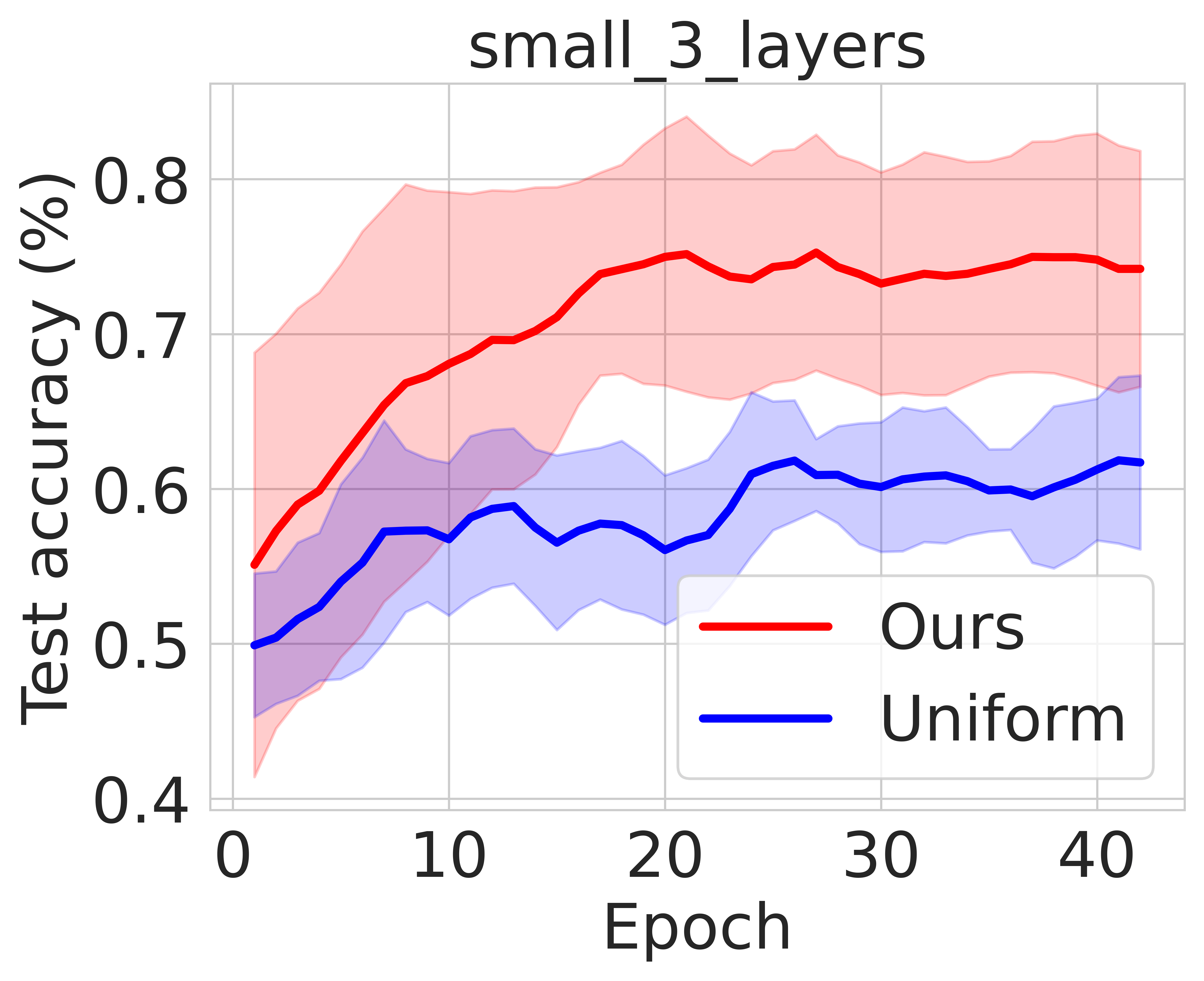}
  \end{minipage}%
  \hfill
  \begin{minipage}[b]{0.33\textwidth}
  \centering
 \includegraphics[width=1\textwidth]{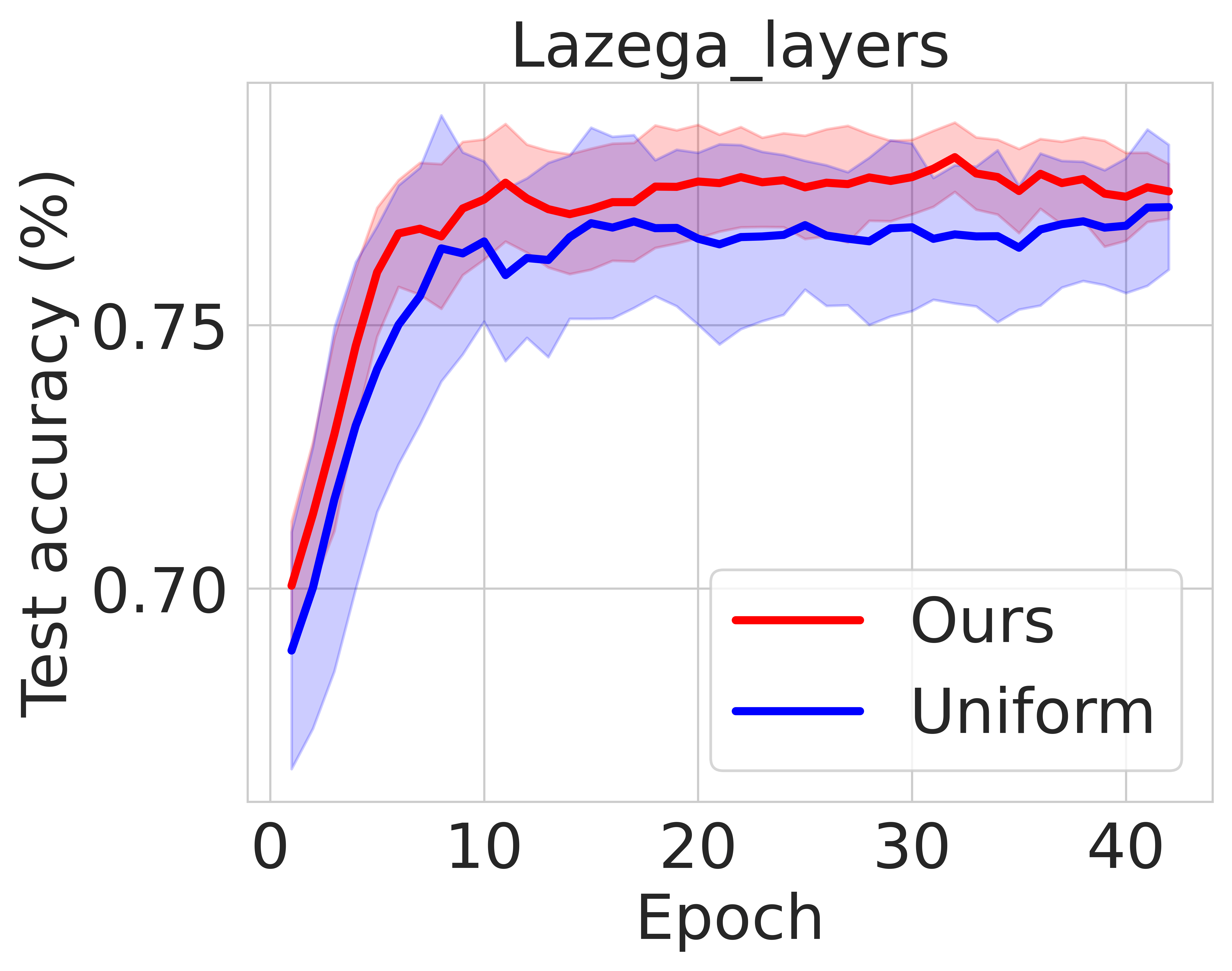}
  \end{minipage}%
  \hfill
      \begin{minipage}[b]{0.33\textwidth}
  \centering
 \includegraphics[width=1\textwidth]{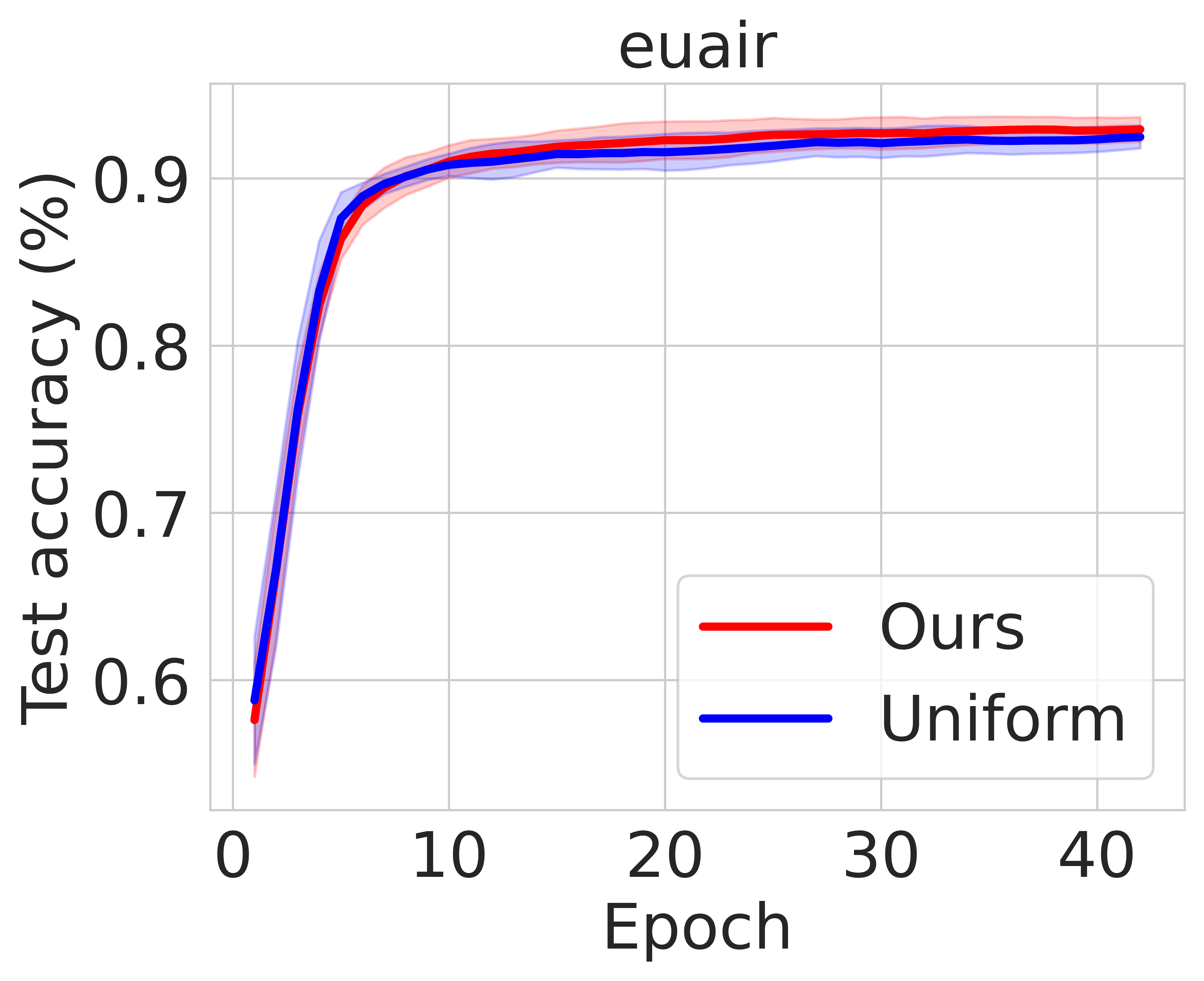}
  \end{minipage}%
  \caption{Test accuracy throughout the training process when sampling neighboring layers ($k$=1) with respect to our approach (red) and uniform sampling (blue) using the loss in Eqn.~\ref{eqn:euclidean-dist}. Shaded regions correspond to values within one standard deviation of the mean.}
	\label{fig:test-acc}
\end{figure*}

\begin{table}
\centering 
\begin{tabular}{c|c|c|c} 
Network & Layers & Nodes & Edges  \\ [0.1ex] 
\hline\vspace{-.25cm}\\[0.1ex] 
rand internship& 6 & 51 & 1382 \\
CKM & 3 & 246 & 1551 \\
Vickers & 3 & 29 & 740 \\
small 3 layers & 3 & 18 & 541  \\
LAZEGA & 3 & 71 & 2223 \\
EU AIR & 37 & 450 & 3588 
\end{tabular}
\caption{Relevant data sets used for the evaluations.}
\label{table:datasets}
\end{table}

\section{Results}
\label{sec:results}

We evaluated our approach in the link prediction setting and compared its performance to that of uniform sampling neighboring layers on the following data sets. The data sets are described below and further information about the core data sets can be found in Table~\ref{table:datasets}. The results were averaged over 10 trials with 5-fold cross validation.


\begin{figure*}
  \centering
  \begin{minipage}[t]{0.33\textwidth}
  \centering
 \includegraphics[width=1\textwidth]{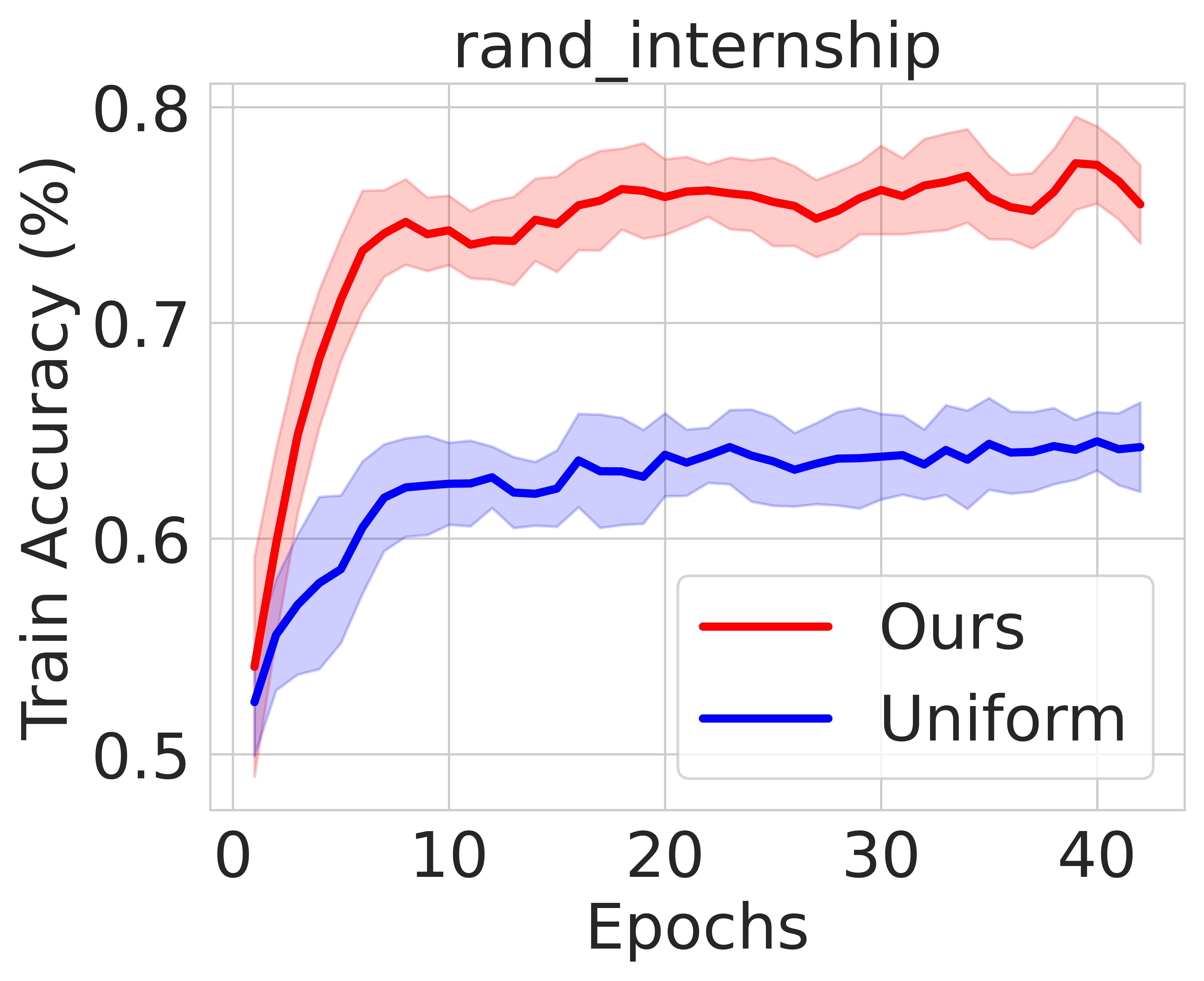}
  \end{minipage}%
  \hfill
  \begin{minipage}[t]{0.33\textwidth}
  \centering
 \includegraphics[width=1\textwidth]{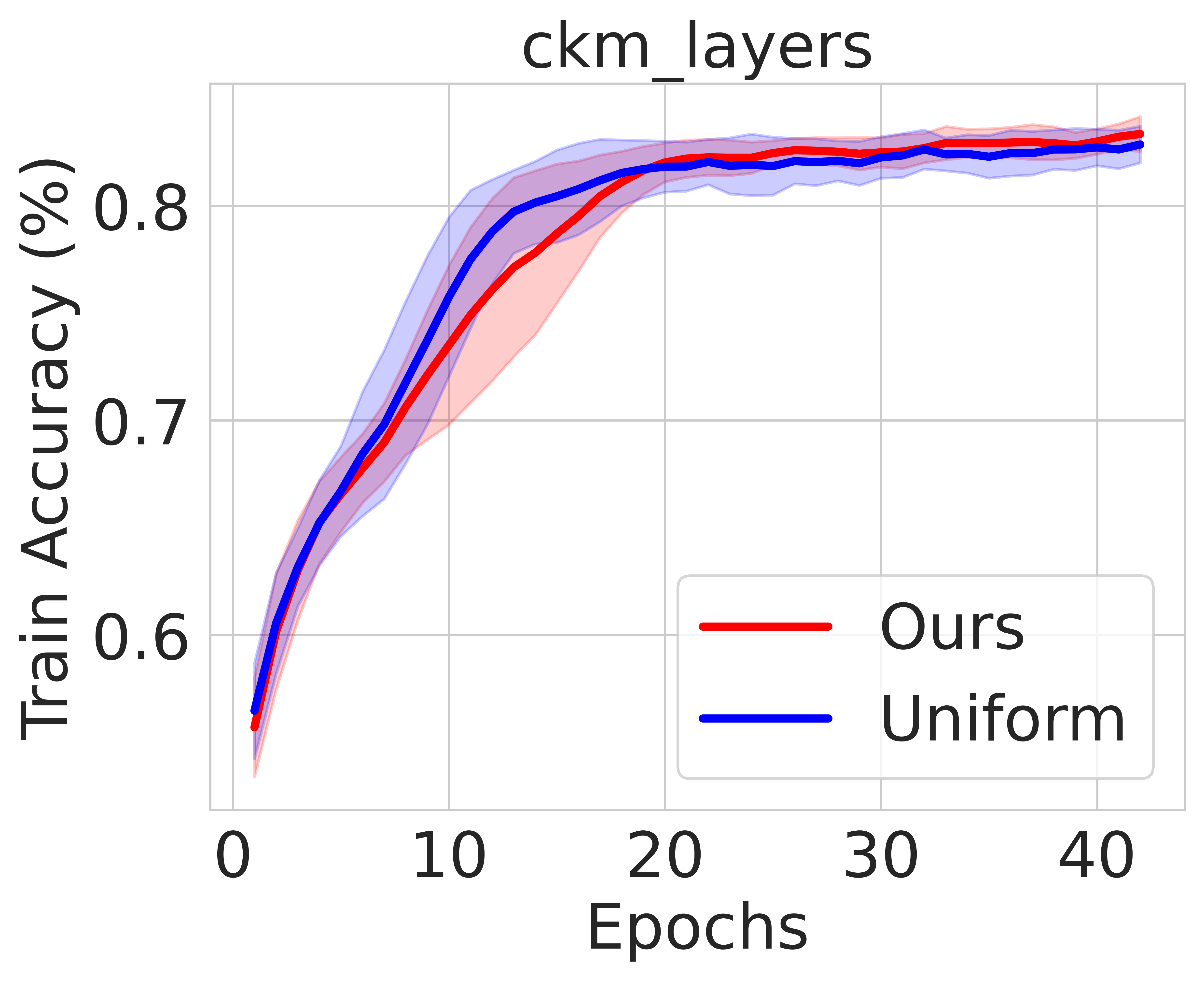}
  \end{minipage}%
  \hfill
    \begin{minipage}[t]{0.33\textwidth}
  \centering
 \includegraphics[width=1\textwidth]{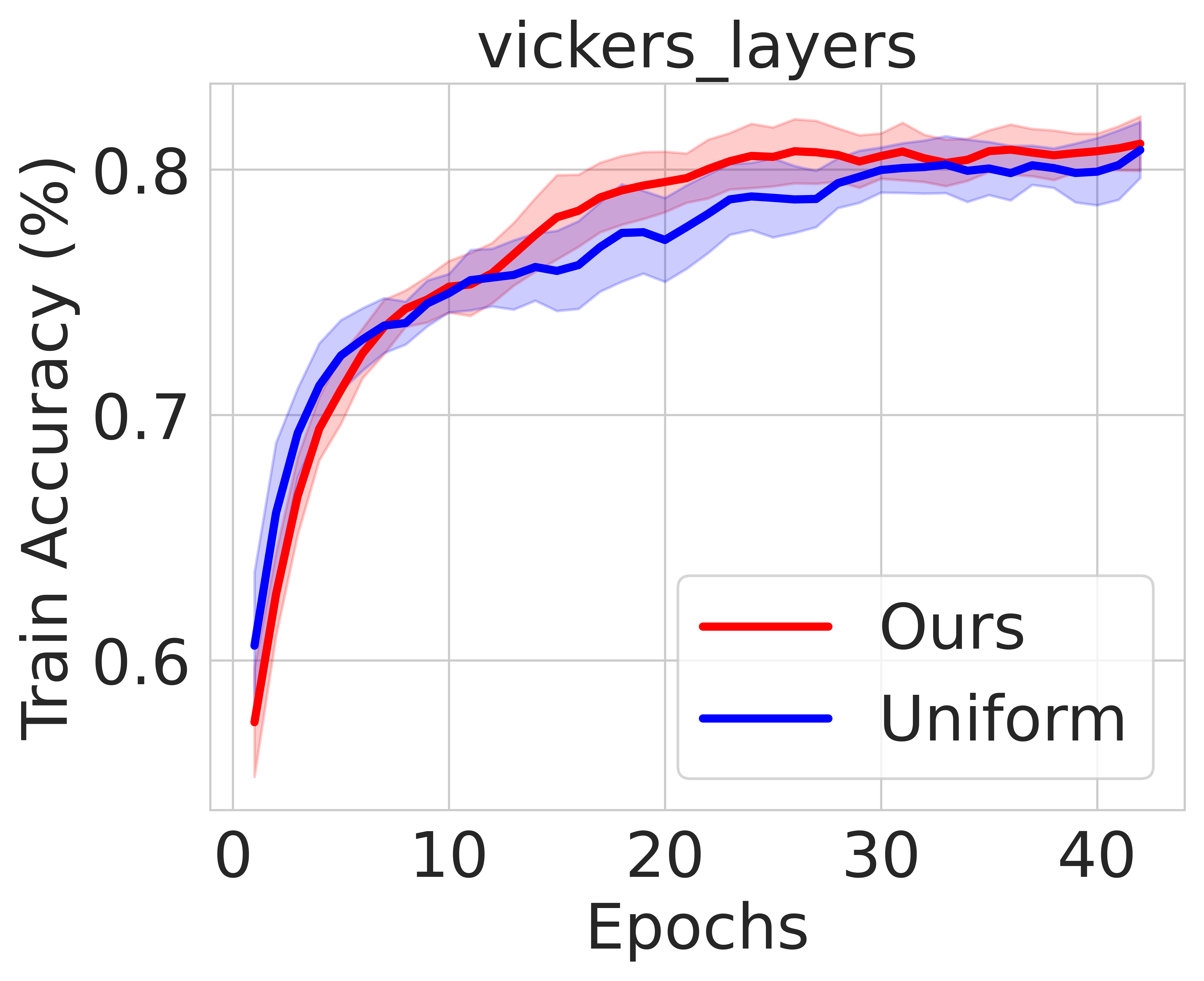}
  \end{minipage}
  
    \begin{minipage}[b]{0.33\textwidth}
  \centering
 \includegraphics[width=1\textwidth]{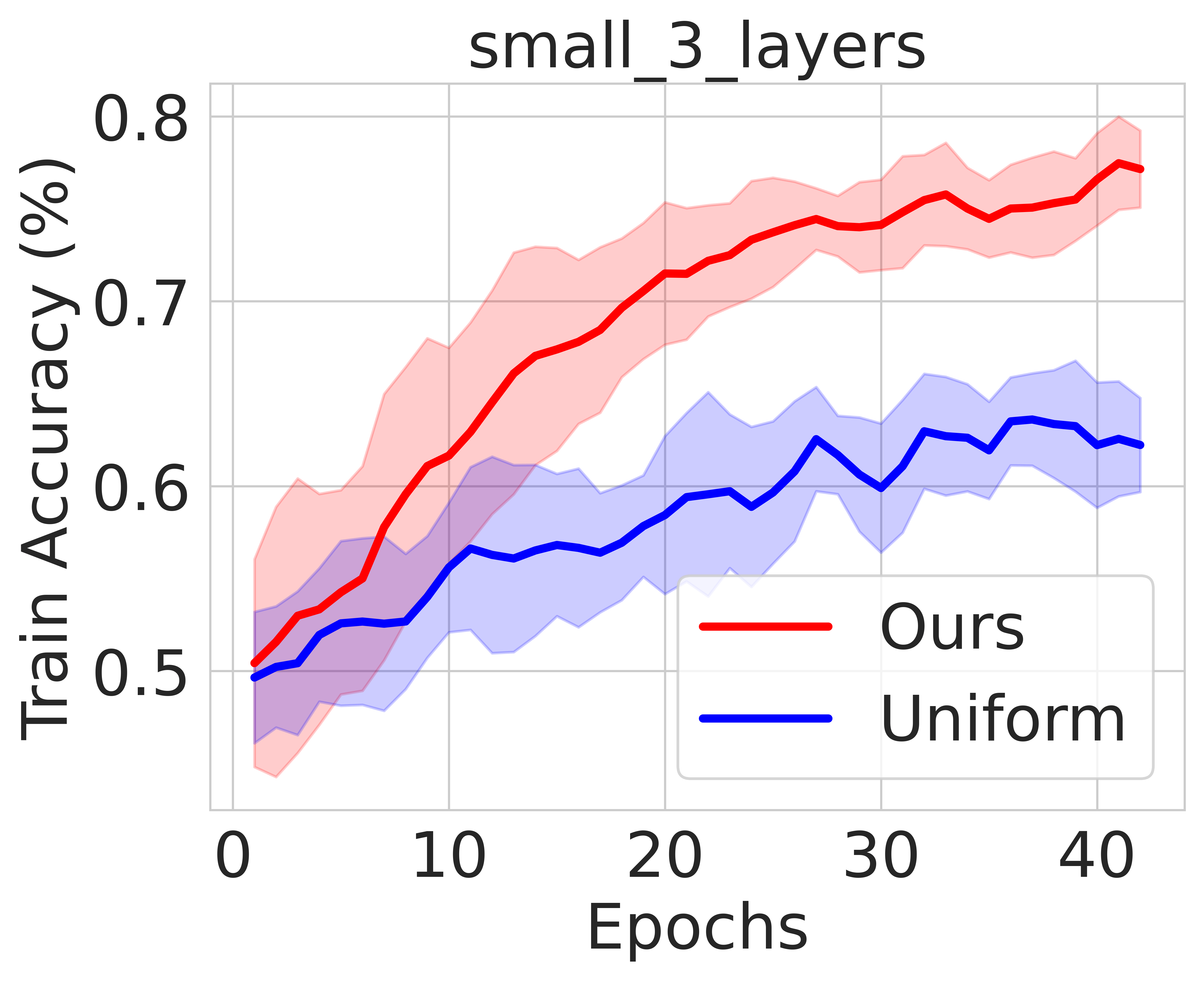}
  \end{minipage}%
  \hfill
  \begin{minipage}[b]{0.33\textwidth}
  \centering
 \includegraphics[width=1\textwidth]{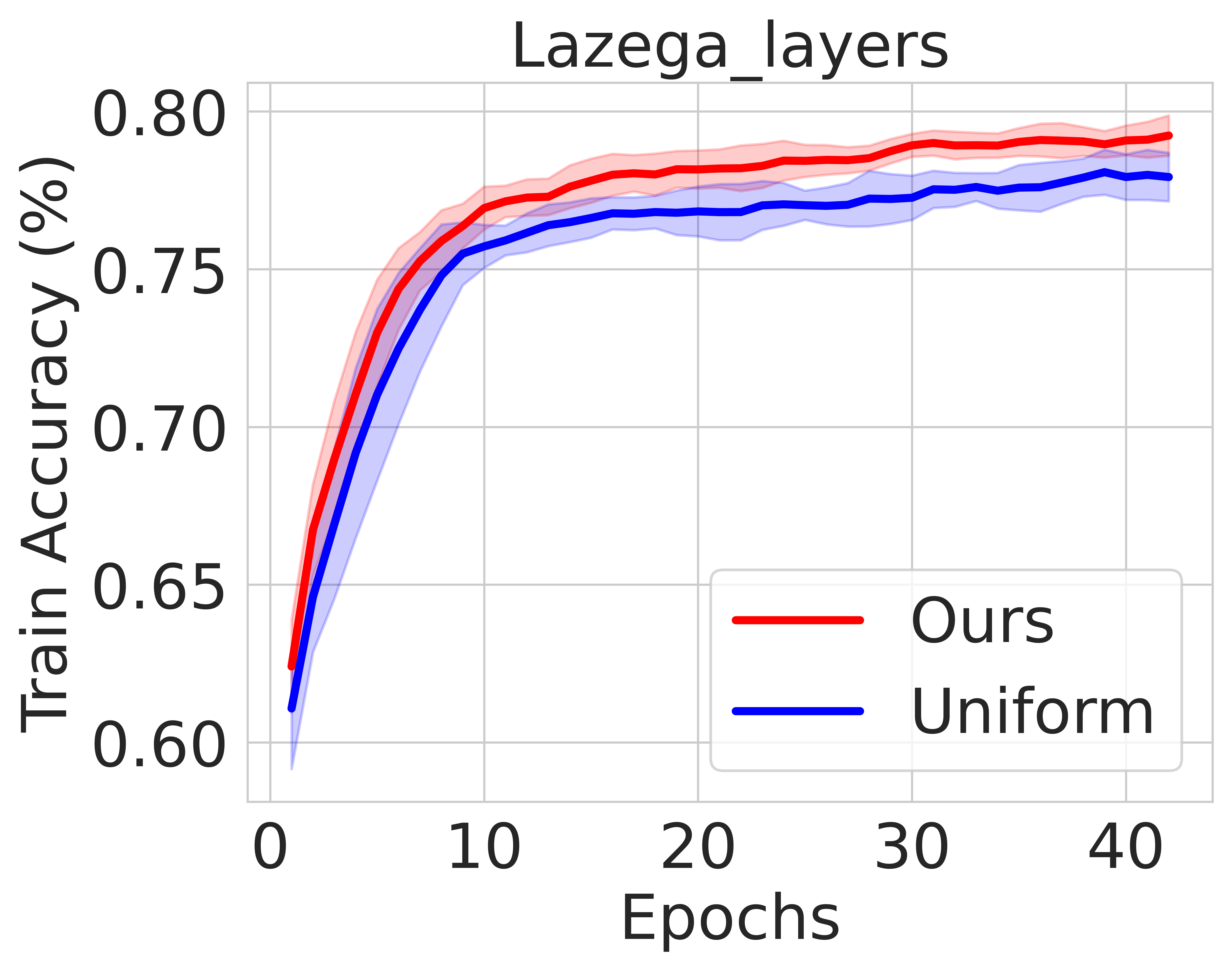}
  \end{minipage}%
  \hfill
      \begin{minipage}[b]{0.33\textwidth}
  \centering
 \includegraphics[width=1\textwidth]{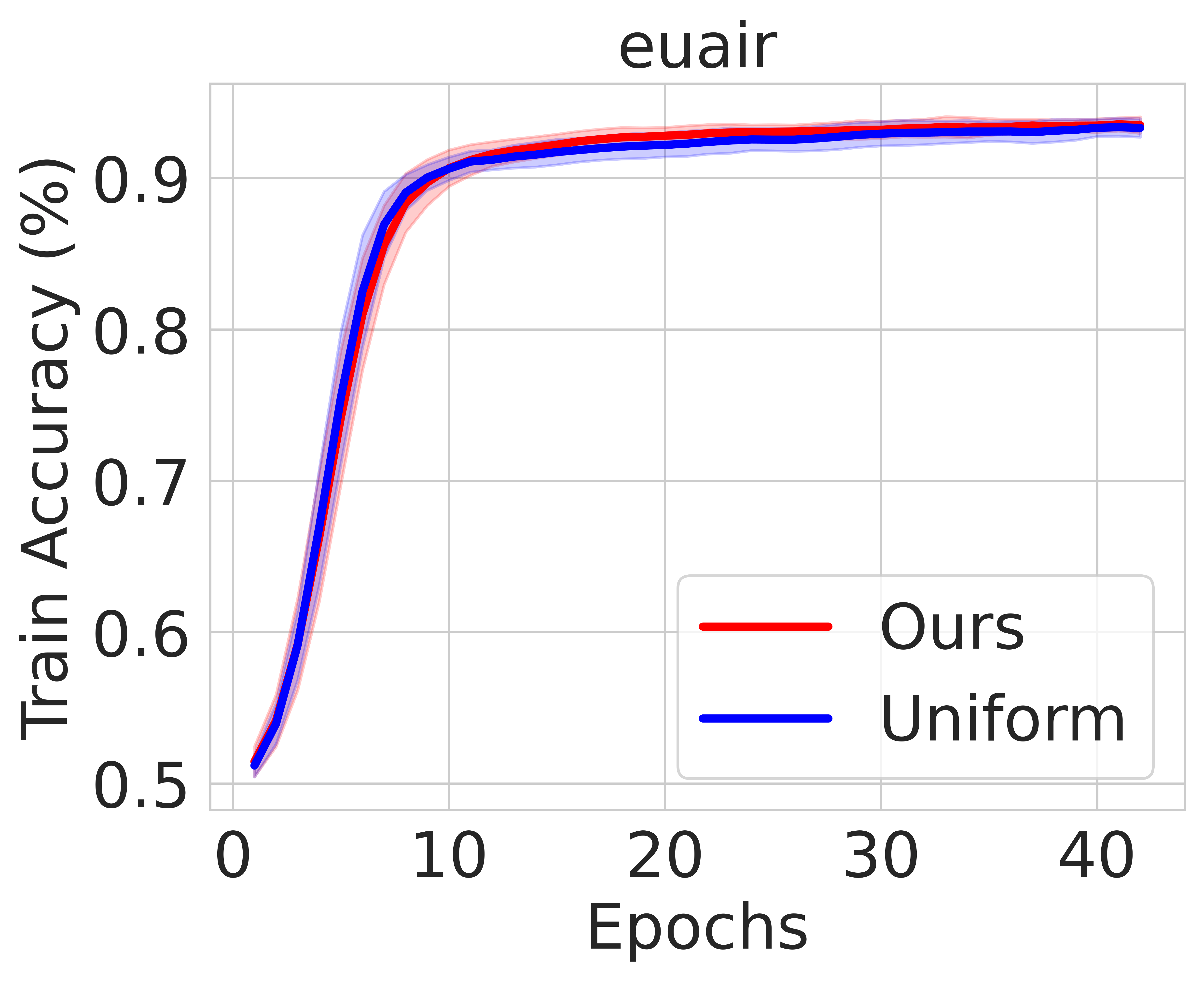}
  \end{minipage}%
  \caption{Train accuracy using the Euclidean distance metric for the loss (Eqn.~\ref{eqn:euclidean-dist}).}
	\label{fig:train-acc}
\end{figure*}

\begin{itemize}
    \item \textbf{rand internship:} a synthetic data set that comprised of 6 layers, where the first layer is taken from the CKM data set, and the remaining 5 layers are generated randomly so that they are of 90\%, 10\%, 10\%, 5\%, and 1\% similar to the first layer
    
    \item \textbf{CKM:} Physicians data collected from towns in Illinois
    
    \item \textbf{Vickers:} Vickers consists Data from seventh grade students who were asked 3 types of questions
    
    \item \textbf{small 3 layers:} a small network where the first layer is the first layer of the Vickers data set, and the second and third layers are of 90\% and 10\% similarity to the first
    
    \item \textbf{LAZEGA:} a Social network describing 3 relationships between law firm partners and associates
    
    \item \textbf{EUAIR:} contains direct flights between cities for 37 different European airlines
\end{itemize}

\subsection{Euclidean Distance Loss}

In this subsection, we evaluate the performance of our approach  define the similarity between the layers with respect to the Euclidean distance between their embeddings. That is, for sampling the neighbors of layer $i \in [L]$, the loss of a neighboring layer $j \in [n]$ at time step $t$ is given by
\begin{equation}
    \label{eqn:euclidean-dist}
\ell_{t, j} = \norm{v_{i}^{(t)} - v_{j}^{(t)}} \quad \forall{j \in [n]}.
\end{equation}

Figs.~\ref{fig:test-acc} and \ref{fig:train-acc} depict the results of our evaluations on the data sets enumerated above at each epoch during the training of the networks. In particular, in Fig.~\ref{fig:test-acc}, we observe the test accuracy of the multiplex network at each epoch during the training with uniform sampling (blue) and our method (red) of sampling the relevant layers. From Fig.~\ref{fig:test-acc} we note that our approach performs uniformly better in test accuracy relative to uniform sampling, and at times, by a very significant margin -- up to +15\% test accuracy (absolute terms). In particular, for the synthetic examples we constructed (rand-internship and small-3-layers) where there is a significant amount of similarity between certain layers, our algorithm is able to adaptively learn and sample the most relevant layers for training. Interestingly, for real-world data sets, a more modest, but significant improvement in test accuracy exists relative to uniform sampling. For example, we outperform uniform sampling by up to 8\% (absolute terms) test accuracy for the VICKERS data set in early stages of the training, and consistently outperform uniform sampling on the LAZEGA and EUAIR data sets across all training iterations.

In addition to the boosts in the test accuracy of the network, we also considered the training accuracy of our approach across the varying training iterations in Fig.~\ref{fig:train-acc}. Here, we see a similar trend as before: our approach uniformly outperforms uniform sampling on all data sets in context of the entire training regime. In fact, for the train accuracy, our approach leads to arguably more significant gains as evidence by the results on, e.g., small-3-layers, where the gap in performance is over 15\% in absolute accuracy.

\subsection{Cosine Similarity Loss}

In the previous subsection, we used the Euclidean distance between the embeddings quantify the loss of the neighboring layers. Next, we evaluate the performance of our approach using the cosine similarity metric where for each layer $i \in [L]$ the loss of each neighbor $j \in [n]$ is defined as 
\begin{equation}
    \label{eqn:cosine-sim}
    \ell_{t, j} = 1 - \nicefrac{\dotp{v_{i}^{(t)}}{v_{j}^{(t)}}}{\norm{v_{i}^{(t)}}\norm{v_{j}^{(t)}}}.
\end{equation}
Fig.~\ref{fig:test-acc-cosine-sim} depicts the results of our evaluations. Here, we can see a similar trend in performance as in the case of Euclidean loss, except the performance of the cosine similarity metric seems to be slightly better in some of the real-world data sets.

\begin{figure*}
  \centering
  \begin{minipage}[t]{0.33\textwidth}
  \centering
 \includegraphics[width=1\textwidth]{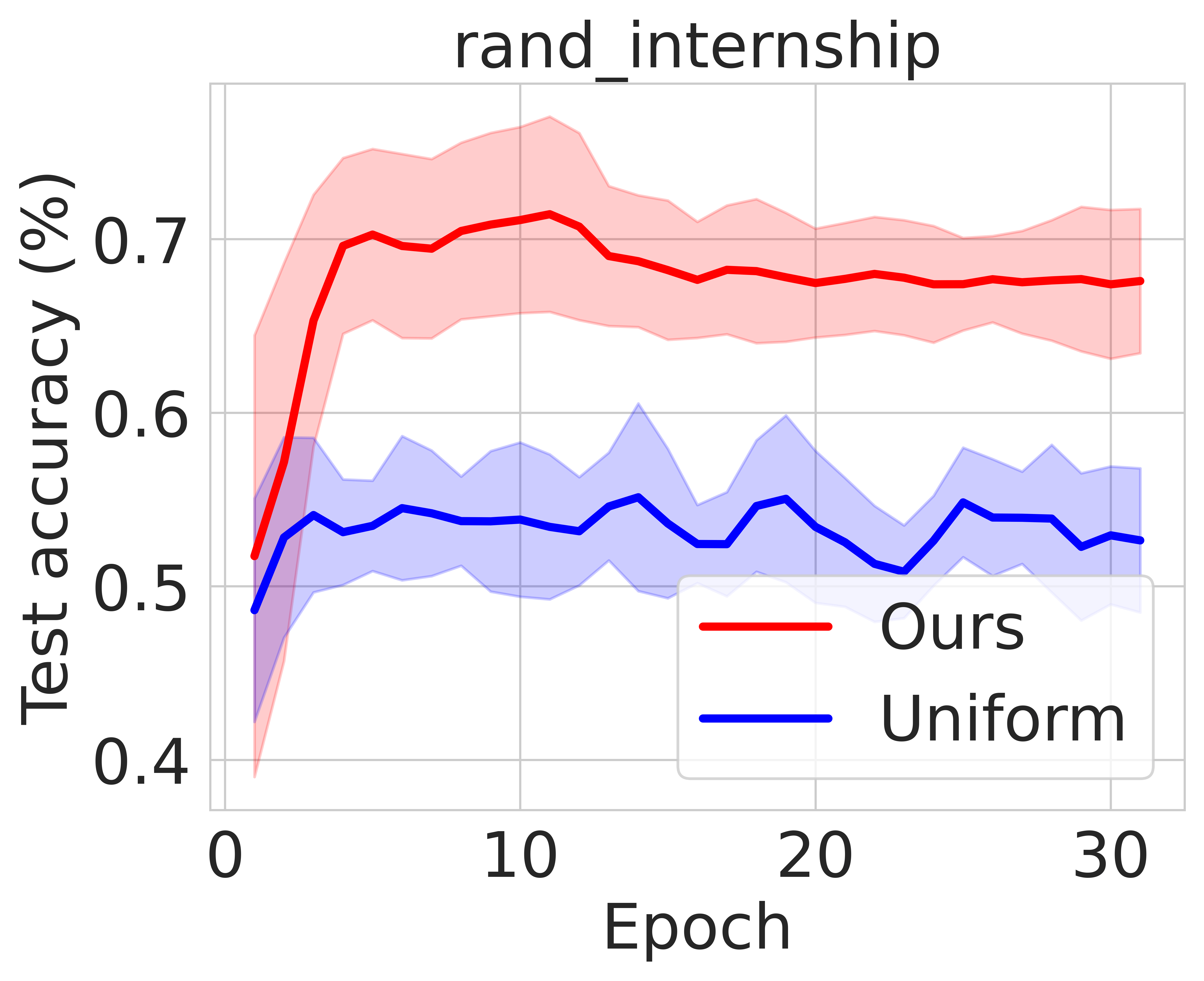}
  \end{minipage}%
  \hfill
  \begin{minipage}[t]{0.33\textwidth}
  \centering
 \includegraphics[width=1\textwidth]{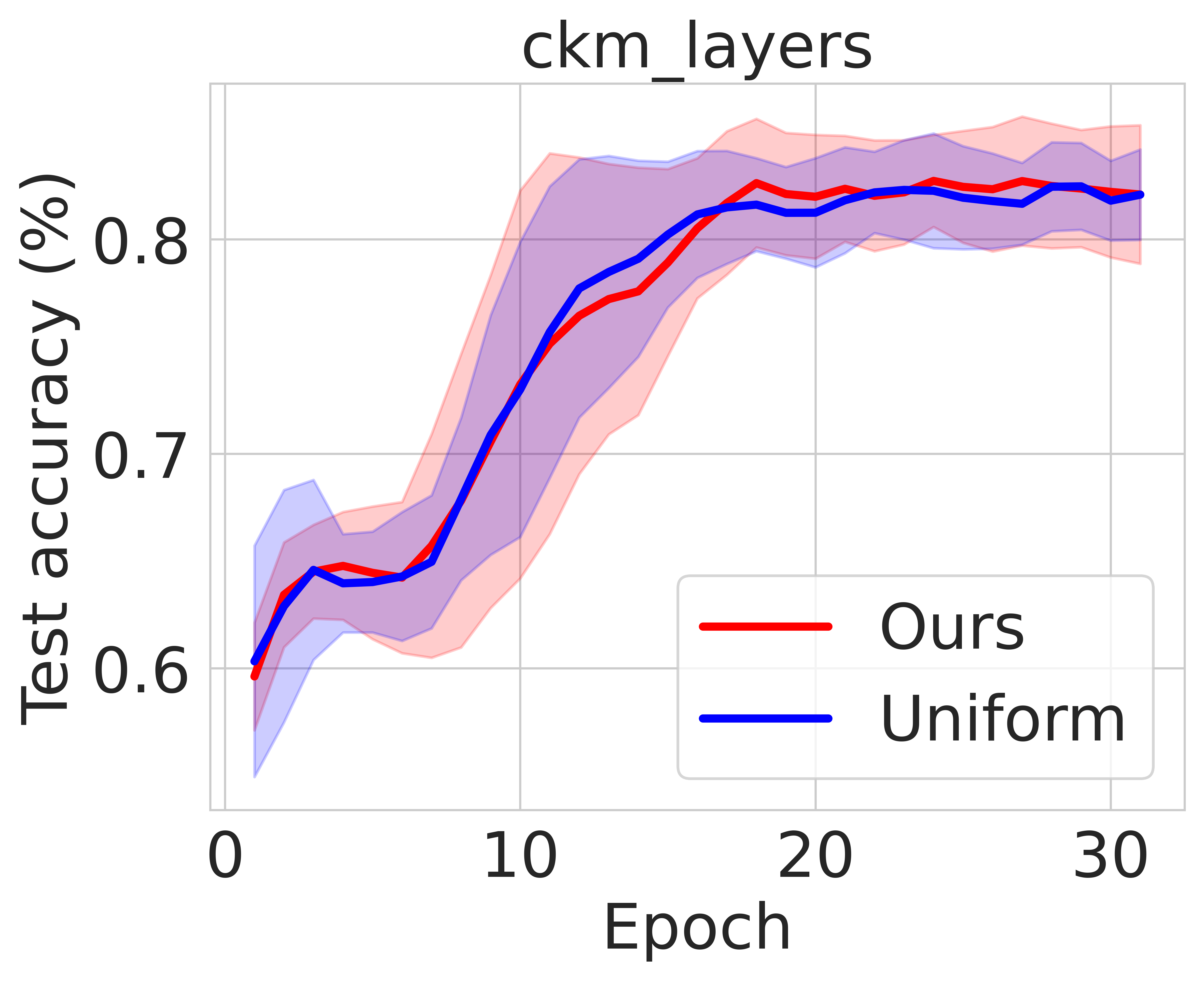}
  \end{minipage}%
  \hfill
    \begin{minipage}[t]{0.33\textwidth}
  \centering
 \includegraphics[width=1\textwidth]{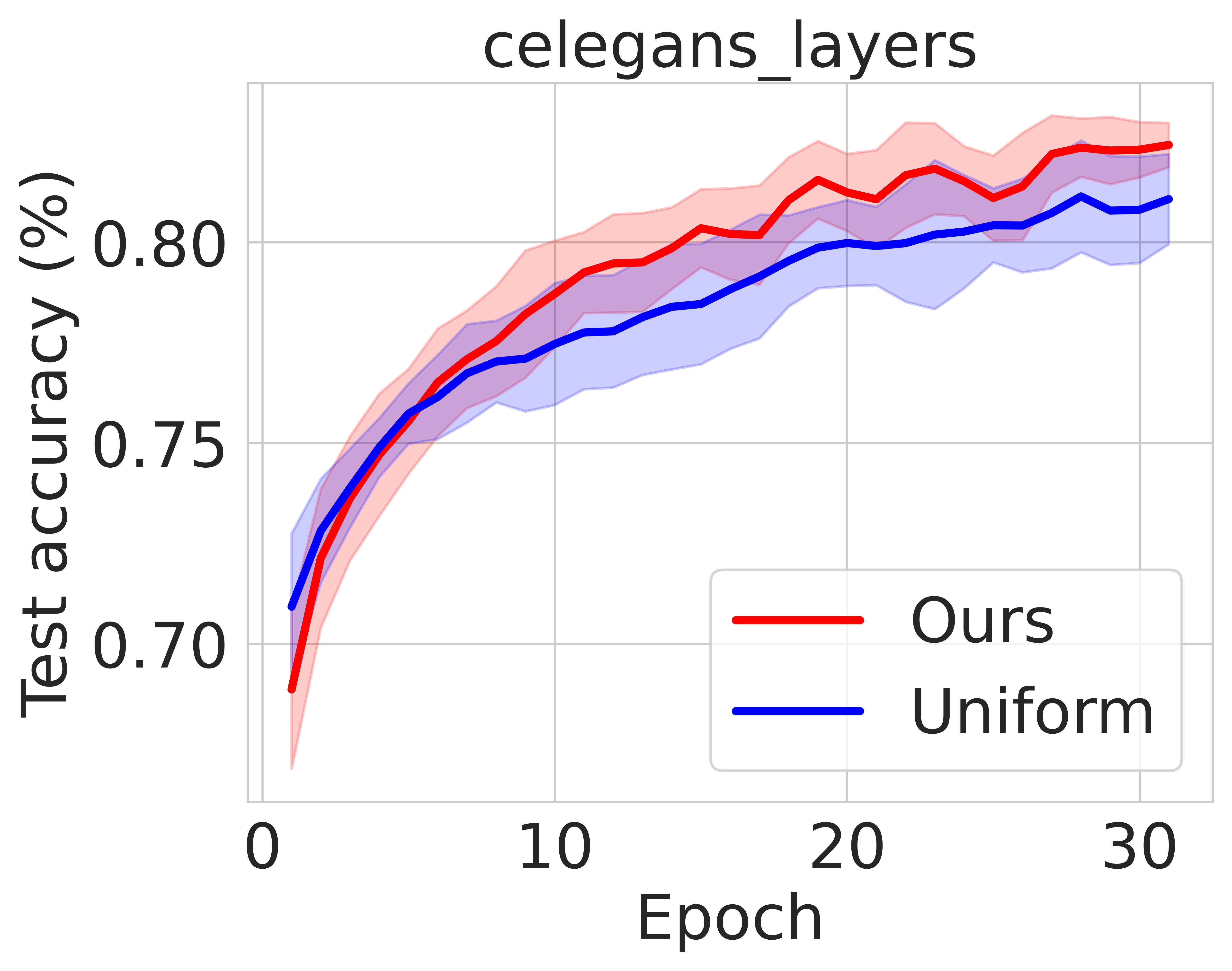}
  \end{minipage}
  
    \begin{minipage}[b]{0.33\textwidth}
  \centering
 \includegraphics[width=1\textwidth]{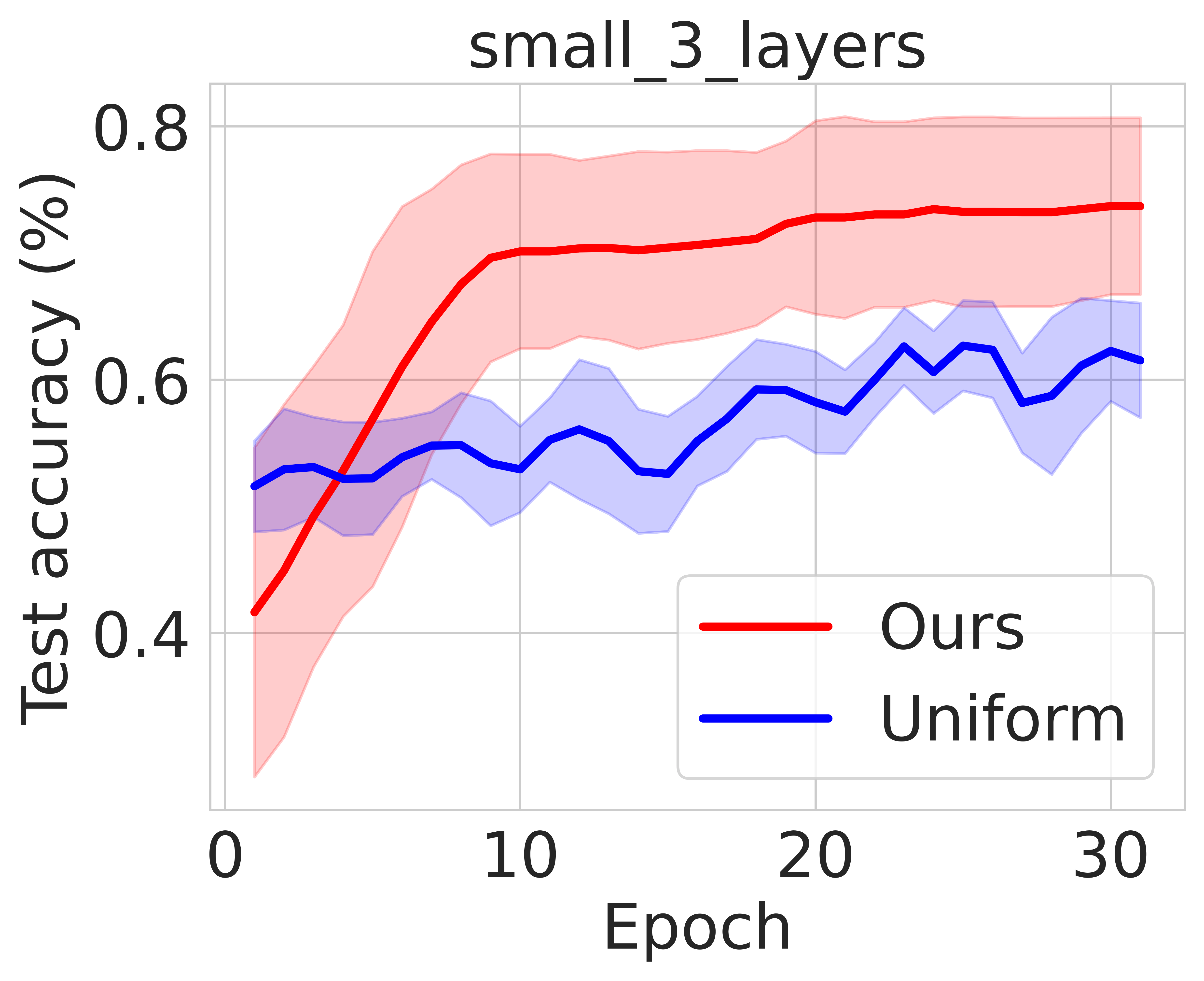}
  \end{minipage}%
  \hfill
  \begin{minipage}[b]{0.33\textwidth}
  \centering
 \includegraphics[width=1\textwidth]{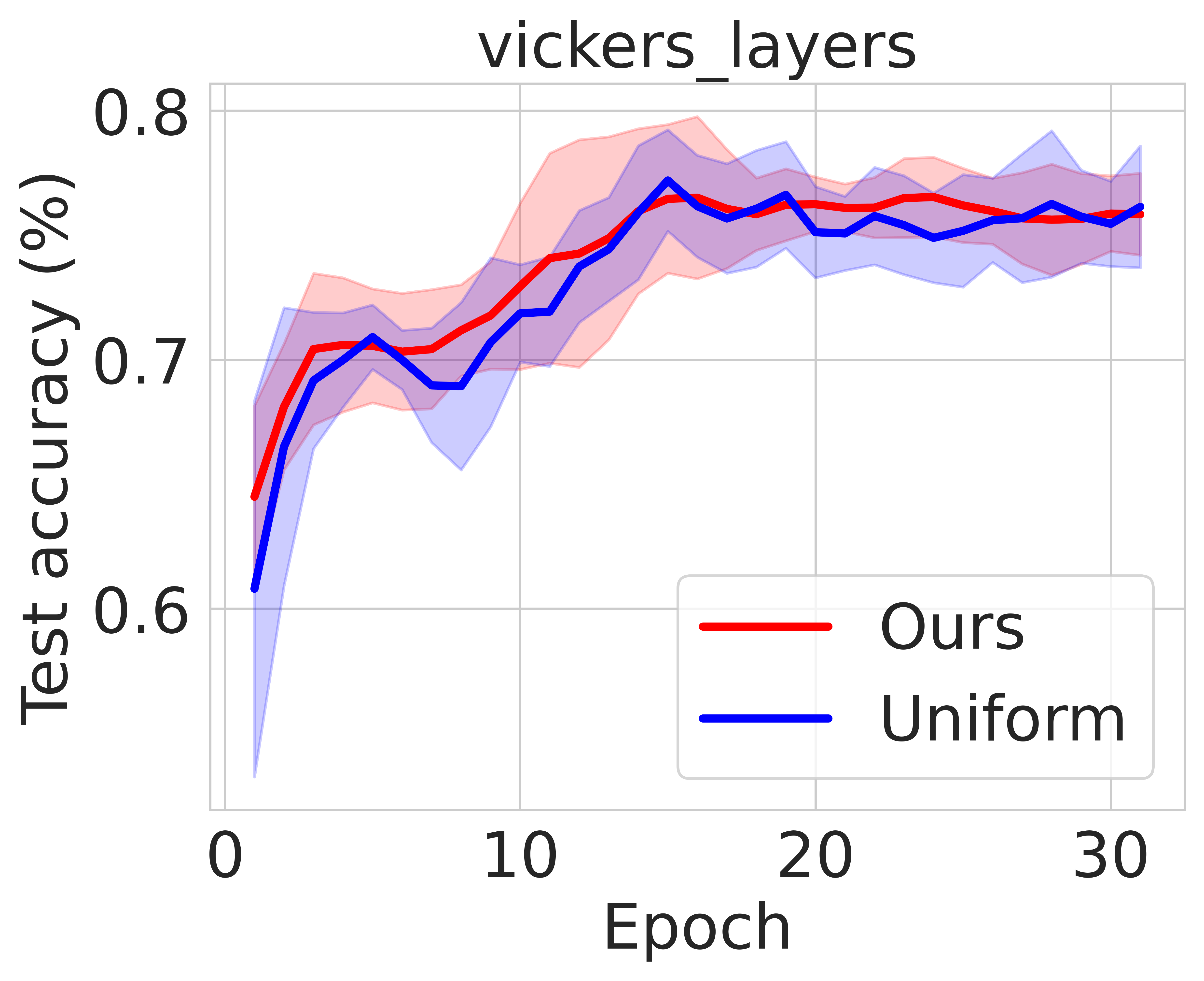}
  \end{minipage}%
  \hfill
      \begin{minipage}[b]{0.33\textwidth}
  \centering
\includegraphics[width=1\textwidth]{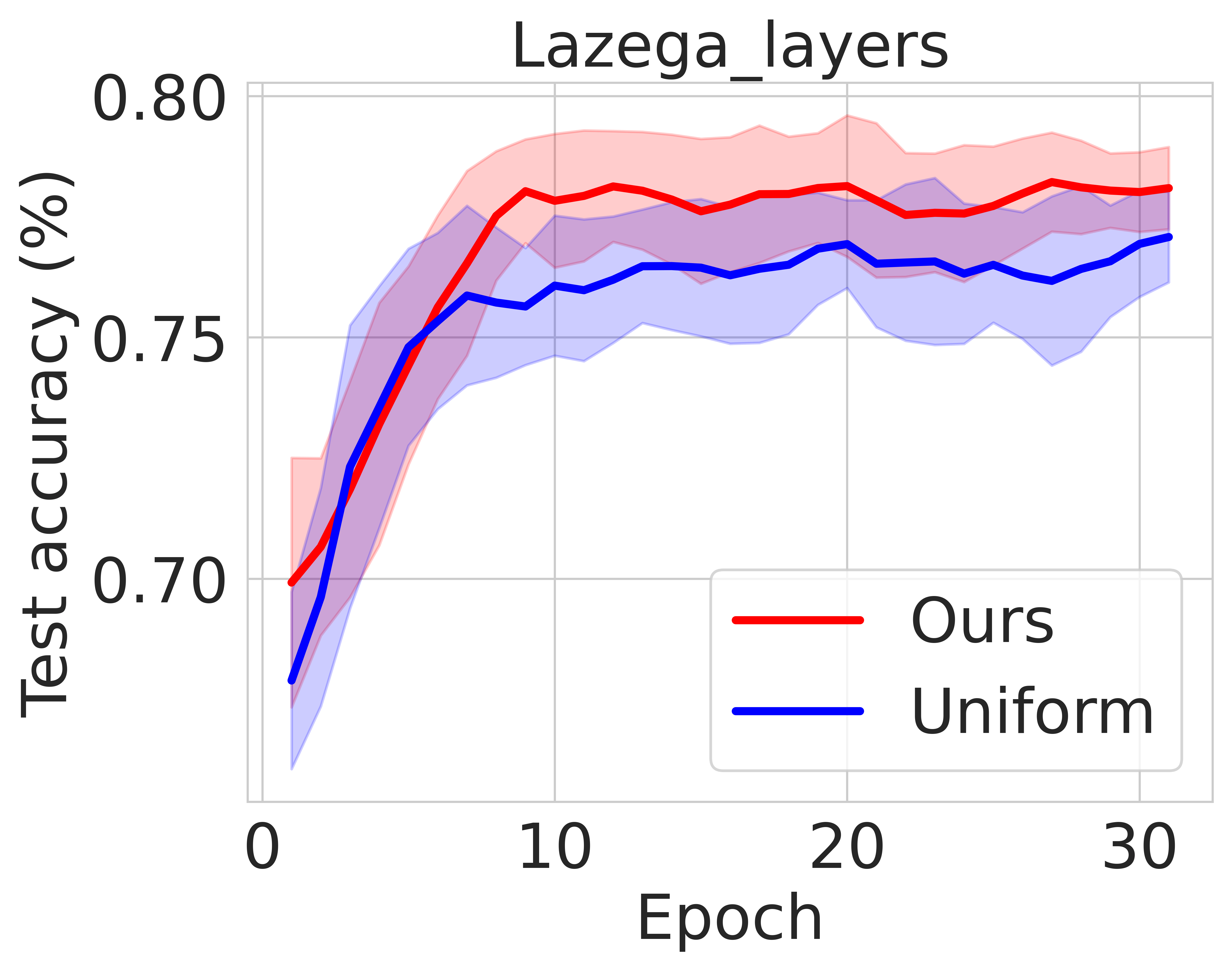}
  \end{minipage}%
  \caption{Test accuracy using the Cosine Similarity Metric (Eqn.~\ref{eqn:cosine-sim}).}
	\label{fig:test-acc-cosine-sim}
\end{figure*}

\section{Discussion}
\label{sec:discussion}

Overall, our results in Sec.~\ref{sec:discussion} suggest that our bandit sampling approach has the potential to lead to significant gains over uniform sampling on both synthetic and real-world data sets.

\subsection{Open Problems}

\begin{enumerate}
    \item Extending our algorithm and the corresponding regret analysis to sampling multiple layers (i.e., $k > 1$). Our ongoing work suggests that this can be done using the dependent randomized rounding  procedure to sample $k$ layers with probabilities $p_t$ such that $\sum_{i \in [n]} p_{ti} = k$ and the technique of~\cite{uchiya2010algorithms}.
    
    \item Balancing the trade-off between sampling layers and the amount of time spent training the sampled layer.
    
    \item Additional experimentation with varying losses and analysis into the class of appropriate loss functions. The losses considered in the paper account for proximity in the embedding space, but it would be ideal to specify a loss for classification or prediction performance.
\end{enumerate}

\section{Conclusion}
\label{sec:conclusion}

In this work, we proposed a novel layer-sampling method to accelerate the train and prediction performance of layer multiplex networks with many layers. Unlike prior work, our approach adaptively learns a sampling distribution over the neighbors of each layer so that only the layers with relevant information will be sampled and used in training and prediction. Unlike prior work in scalable multiplex learning, our approach has potential to alleviate the training and inference-time requirements of large multiplex networks with a large number of layers by focusing only on the relevant information at each time step. We presented empirical evaluations that support the practical effectiveness of our approach on a variety of synthetic and real-world data sets.

\paragraph{Acknowledgments} This research was supported by JPMorgan Chase \& Co.

\paragraph{Disclaimer} This paper was prepared for information purposes by the AI Research
Group of JPMorgan Chase \& Co and its affiliates (``J.P. Morgan''),
and is not a product of the Research Department of J.P. Morgan. J.P.
Morgan makes no explicit or implied representation and warranty and
accepts no liability, for the completeness, accuracy or reliability of
information, or the legal, compliance, financial, tax or accounting
effects of matters contained herein. This document is not intended
as investment research or investment advice, or a recommendation,
offer or solicitation for the purchase or sale of any security, financial
instrument, financial product or service, or to be used in any way for
evaluating the merits of participating in any transaction.

\afterpage{\FloatBarrier}

\bibliographystyle{named}
\bibliography{refs}

\begin{thebibliography}{}

\bibitem[\protect\citeauthoryear{Adamic and Adar}{2003}]{adamic:01}
Lada~A. Adamic and Eytan Adar.
\newblock Friends and neighbors on the web.
\newblock {\em Social Networks}, 25(3):211--230, 2003.

\bibitem[\protect\citeauthoryear{Auer \bgroup \em et al.\egroup
  }{2002}]{auer2002finite}
Peter Auer, Nicolo Cesa-Bianchi, and Paul Fischer.
\newblock Finite-time analysis of the multiarmed bandit problem.
\newblock {\em Machine learning}, 47(2):235--256, 2002.

\bibitem[\protect\citeauthoryear{Bruna \bgroup \em et al.\egroup
  }{2013}]{bruna2013spectral}
Joan Bruna, Wojciech Zaremba, Arthur Szlam, and Yann LeCun.
\newblock Spectral networks and locally connected networks on graphs.
\newblock {\em arXiv preprint arXiv:1312.6203}, 2013.

\bibitem[\protect\citeauthoryear{Bubeck and
  Cesa-Bianchi}{2012}]{bubeck2012regret}
S{\'e}bastien Bubeck and Nicolo Cesa-Bianchi.
\newblock Regret analysis of stochastic and nonstochastic multi-armed bandit
  problems.
\newblock {\em arXiv preprint arXiv:1204.5721}, 2012.

\bibitem[\protect\citeauthoryear{Cesa-Bianchi and
  Lugosi}{2006}]{cesa2006prediction}
Nicolo Cesa-Bianchi and G{\'a}bor Lugosi.
\newblock {\em Prediction, learning, and games}.
\newblock Cambridge university press, 2006.

\bibitem[\protect\citeauthoryear{Cesa-Bianchi \bgroup \em et al.\egroup
  }{2007}]{cesa2007improved}
Nicolo Cesa-Bianchi, Yishay Mansour, and Gilles Stoltz.
\newblock Improved second-order bounds for prediction with expert advice.
\newblock {\em Machine Learning}, 66(2):321--352, 2007.

\bibitem[\protect\citeauthoryear{Duvenaud \bgroup \em et al.\egroup
  }{2015}]{duvenaud2015convolutional}
David~K Duvenaud, Dougal Maclaurin, Jorge Iparraguirre, Rafael Bombarell,
  Timothy Hirzel, Al{\'a}n Aspuru-Guzik, and Ryan~P Adams.
\newblock Convolutional networks on graphs for learning molecular fingerprints.
\newblock In {\em Advances in neural information processing systems}, pages
  2224--2232, 2015.

\bibitem[\protect\citeauthoryear{Gaillard \bgroup \em et al.\egroup
  }{2014}]{gaillard2014second}
Pierre Gaillard, Gilles Stoltz, and Tim Van~Erven.
\newblock A second-order bound with excess losses.
\newblock In {\em Conference on Learning Theory}, pages 176--196. PMLR, 2014.

\bibitem[\protect\citeauthoryear{Gilmer \bgroup \em et al.\egroup
  }{2017}]{gilmer2017neural}
Justin Gilmer, Samuel~S Schoenholz, Patrick~F Riley, Oriol Vinyals, and
  George~E Dahl.
\newblock Neural message passing for quantum chemistry.
\newblock In {\em Proceedings of the 34th International Conference on Machine
  Learning-Volume 70}, pages 1263--1272. JMLR. org, 2017.

\bibitem[\protect\citeauthoryear{Grover and
  Leskovec}{2016}]{grover2016node2vec}
Aditya Grover and Jure Leskovec.
\newblock node2vec: Scalable feature learning for networks.
\newblock In {\em Proceedings of the 22nd ACM SIGKDD international conference
  on Knowledge discovery and data mining}, pages 855--864, 2016.

\bibitem[\protect\citeauthoryear{Katz}{1953}]{katz:53}
Leo Katz.
\newblock A new status index derived from sociometric analysis.
\newblock {\em Psychometrika}, 18(1):39--43, 1953.

\bibitem[\protect\citeauthoryear{Kipf and Welling}{2016}]{kipf2016variational}
Thomas~N Kipf and Max Welling.
\newblock Variational graph auto-encoders.
\newblock {\em arXiv preprint arXiv:1611.07308}, 2016.

\bibitem[\protect\citeauthoryear{Koolen and Van~Erven}{2015}]{koolen2015second}
Wouter~M Koolen and Tim Van~Erven.
\newblock Second-order quantile methods for experts and combinatorial games.
\newblock In {\em Conference on Learning Theory}, pages 1155--1175. PMLR, 2015.

\bibitem[\protect\citeauthoryear{Liben-Nowell and
  Kleinberg}{2007}]{libennowell-04}
David Liben-Nowell and Jon Kleinberg.
\newblock The link-prediction problem for social networks.
\newblock {\em Journal of the American Society for Information Science and
  Technology}, 58(7):1019--1031, 2007.

\bibitem[\protect\citeauthoryear{Liu \bgroup \em et al.\egroup
  }{2020}]{liu2020bandit}
Ziqi Liu, Zhengwei Wu, Zhiqiang Zhang, Jun Zhou, Shuang Yang, Le~Song, and Yuan
  Qi.
\newblock Bandit samplers for training graph neural networks.
\newblock {\em arXiv preprint arXiv:2006.05806}, 2020.

\bibitem[\protect\citeauthoryear{Luo and Schapire}{2014}]{luo2014drifting}
Haipeng Luo and Robert~E Schapire.
\newblock A drifting-games analysis for online learning and applications to
  boosting.
\newblock {\em Advances in Neural Information Processing Systems},
  27:1368--1376, 2014.

\bibitem[\protect\citeauthoryear{Luo and Schapire}{2015}]{luo2015achieving}
Haipeng Luo and Robert~E Schapire.
\newblock Achieving all with no parameters: Adanormalhedge.
\newblock In {\em Conference on Learning Theory}, pages 1286--1304, 2015.

\bibitem[\protect\citeauthoryear{Orabona}{2019}]{orabona2019modern}
Francesco Orabona.
\newblock A modern introduction to online learning.
\newblock {\em arXiv preprint arXiv:1912.13213}, 2019.

\bibitem[\protect\citeauthoryear{Perozzi \bgroup \em et al.\egroup
  }{2014}]{perozzi2014deepwalk}
Bryan Perozzi, Rami Al-Rfou, and Steven Skiena.
\newblock Deepwalk: Online learning of social representations.
\newblock In {\em Proceedings of the 20th ACM SIGKDD international conference
  on Knowledge discovery and data mining}, pages 701--710, 2014.

\bibitem[\protect\citeauthoryear{Potluru \bgroup \em et al.\egroup
  }{2020}]{potlurudeeplex}
Vamsi~K Potluru, Robert~E Tillman, Prashant Reddy, and Manuela Veloso.
\newblock Deeplex: A gnn for link prediction in multiplex networks.
\newblock {\em SIAM Workshop on Network Science}, 2020.

\bibitem[\protect\citeauthoryear{Sani \bgroup \em et al.\egroup
  }{2014}]{sani2014exploiting}
Amir Sani, Gergely Neu, and Alessandro Lazaric.
\newblock Exploiting easy data in online optimization.
\newblock {\em Advances in Neural Information Processing Systems}, 27:810--818,
  2014.

\bibitem[\protect\citeauthoryear{Tang \bgroup \em et al.\egroup
  }{2015}]{tang2015line}
Jian Tang, Meng Qu, Mingzhe Wang, Ming Zhang, Jun Yan, and Qiaozhu Mei.
\newblock Line: Large-scale information network embedding.
\newblock In {\em Proceedings of the 24th international conference on world
  wide web}, pages 1067--1077, 2015.

\bibitem[\protect\citeauthoryear{Uchiya \bgroup \em et al.\egroup
  }{2010}]{uchiya2010algorithms}
Taishi Uchiya, Atsuyoshi Nakamura, and Mineichi Kudo.
\newblock Algorithms for adversarial bandit problems with multiple plays.
\newblock In {\em International Conference on Algorithmic Learning Theory},
  pages 375--389. Springer, 2010.

\bibitem[\protect\citeauthoryear{Yun \bgroup \em et al.\egroup
  }{2021}]{yun2021neo}
Seongjun Yun, Seoyoon Kim, Junhyun Lee, Jaewoo Kang, and Hyunwoo~J Kim.
\newblock Neo-gnns: Neighborhood overlap-aware graph neural networks for link
  prediction.
\newblock {\em Advances in Neural Information Processing Systems}, 34, 2021.

\bibitem[\protect\citeauthoryear{Zhang and Chen}{2018}]{zhang2018link}
Muhan Zhang and Yixin Chen.
\newblock Link prediction based on graph neural networks.
\newblock In {\em Advances in Neural Information Processing Systems}, pages
  5165--5175, 2018.

\bibitem[\protect\citeauthoryear{Zhang \bgroup \em et al.\egroup
  }{2018}]{zhang2018scalable}
Hongming Zhang, Liwei Qiu, Lingling Yi, and Yangqiu Song.
\newblock Scalable multiplex network embedding.
\newblock In {\em IJCAI}, volume~18, pages 3082--3088, 2018.

\bibitem[\protect\citeauthoryear{Zhang \bgroup \em et al.\egroup
  }{2021}]{zhang2021biased}
Qingru Zhang, David Wipf, Quan Gan, and Le~Song.
\newblock A biased graph neural network sampler with near-optimal regret.
\newblock {\em arXiv preprint arXiv:2103.01089}, 2021.

\end{thebibliography}

\end{document}